\documentclass[lettersize, journal]{IEEEtran}

\IEEEoverridecommandlockouts                              %

\usepackage[noadjust]{cite}

\usepackage{graphics} %
\usepackage{graphicx}
\usepackage{color}
\usepackage{xcolor}
\usepackage{amsmath} %
\usepackage{amssymb}  %
\usepackage{amsthm}
\usepackage{mathtools}
\usepackage{verbatim}
\usepackage{soul}
\usepackage{multirow}
\usepackage{array}   %
\usepackage{makecell} %

\usepackage{pifont}
\usepackage{tikz}

\newcommand*\circled[1]{\tikz[baseline=(char.base)]{
    \node[shape=circle,draw,inner sep=0.5pt] (char) {{\small #1}};}}

\usepackage[hidelinks]{hyperref}
\usepackage[bottom]{footmisc} %

\usepackage{caption}
\usepackage{subcaption}
\usepackage[ruled, linesnumbered, vlined]{algorithm2e}
\SetAlFnt{\small}
\SetAlCapFnt{\small}
\SetAlCapNameFnt{\small}

\theoremstyle{definition}
\newtheorem{definition}{Definition}

\usepackage{upgreek}
\usepackage{bm}

\definecolor{colorfor1}{HTML}{6da4f3}
\definecolor{colorfor2}{HTML}{8ec18b}
\definecolor{colorfor3}{HTML}{ea6e68}

\definecolor{significant_001}{HTML}{68c79c}
\definecolor{significant_01}{HTML}{4f9a90}
\definecolor{significant_05}{HTML}{fa8c63}
\newcommand{\sigsss}[1]{\underline{#1}}
\newcommand{\sigss}[1]{{#1}}
\newcommand{\sigs}[1]{{#1}}

\renewcommand{\b}[1]{\boldsymbol{#1}}

\newcommand{\specialcell}[2][c]{%
  \begin{tabular}[#1]{@{}c@{}}#2\end{tabular}}

\newtheorem{theorem}{Theorem}

\theoremstyle{remark}

\title{\LARGE \bf Topology-Driven Parallel Trajectory Optimization in Dynamic Environments}
\author{Oscar de Groot, Laura Ferranti, Dariu M. Gavrila, Javier Alonso-Mora
\thanks{The authors are with the Dept. of Cognitive Robotics, TU Delft, 2628 CD Delft, The Netherlands. \texttt{Email: o.m.degroot@tudelft.nl}}%
\thanks{This work received support from the Dutch Science Foundation NWO-TTW within the Veni project HARMONIA (18165), and the European Union within the ERC Starting Grant INTERACT (101041863) and the EVENTS project (101069614). Views and opinions expressed are however those of the author(s) only and do not necessarily reflect those of the European Union or European Commission. Neither the European Union nor the granting authority can be held responsible for them.}}
\begin{document}

\maketitle
\thispagestyle{empty}
\pagestyle{empty}

\begin{abstract}
Ground robots navigating in complex, dynamic environments must compute collision-free trajectories to avoid obstacles safely and efficiently. Nonconvex optimization is a popular method to compute a trajectory in real-time. However, these methods often converge to locally optimal solutions and frequently switch between different local minima, leading to inefficient and unsafe robot motion. In this work, we propose a novel topology-driven trajectory optimization strategy for dynamic environments that plans multiple distinct evasive trajectories to enhance the robot's behavior and efficiency. A global planner iteratively generates trajectories in distinct homotopy classes. These trajectories are then optimized by local planners working in parallel. While each planner shares the same navigation objectives, they are locally constrained to a specific homotopy class, meaning each local planner attempts a different evasive maneuver. The robot then executes the feasible trajectory with the lowest cost in a receding horizon manner. We demonstrate, on a mobile robot navigating among pedestrians, that our approach leads to faster trajectories than existing planners.
\end{abstract}

\section{INTRODUCTION} \label{sec:introduction}
\IEEEPARstart{M}{obile} robots are being deployed in increasingly more complex environments, for example, to automate logistics in warehouses~\cite{simon_inside_2019}
or mobility through self-driving cars \cite{walker_self-driving_2019}. However, it remains challenging to safely and efficiently navigate complex dynamic environments.

In dynamic environments, a robot must make both high-level and low-level decisions. High-level decisions involve, for example, choosing the general direction for safely avoiding obstacles (e.g., going left or right). Low-level decisions involve, for example, determining the exact shape of a trajectory that is both collision-free and dynamically feasible. While these decisions operate on separate levels of the planning problem, they are often not differentiated, which can degrade planner performance in terms of time efficiency and safety. 
Existing methods make the high-level decision implicitly~\cite{brito_model_2019,zhu_chance-constrained_2019,everett_motion_2018}, do not distinguish the high-level and low-level decisions \cite{werling_optimal_2010}$-$\cite{karaman_sampling-based_2011}, only consider static obstacles in the high-level decision~\cite{rosmann_integrated_2017} or require a structured environment to make the high-level decision~\cite{ziegler_making_2014,kunz_autonomous_2015, altche_partitioning_2017}. %
We propose a trajectory optimization algorithm that accounts for these two levels of the planning problem explicitly.

Widely used \textit{optimization-based local planners}, such as Model Predictive Control (MPC)~\cite{brito_model_2019}, formulate the motion planning problem as an optimization problem that can efficiently compute trajectories satisfying dynamic and collision avoidance constraints. Optimization-based planners make high-level decisions implicitly through the initialization of the optimization and tuning of the cost function. 
The planner explores only a small set of possible motion plans near the initial guess. An inadequate initial guess may lead to a poor (i.e., slow or non-smooth) trajectory, slow convergence, or infeasibility. When the initial guess is not consistent over multiple planner cycles, the planner can repeatedly switch its high-level decision, leading to indecisive behavior.

\begin{figure}[t]
    \centering
    \includegraphics*[width=0.4\textwidth]{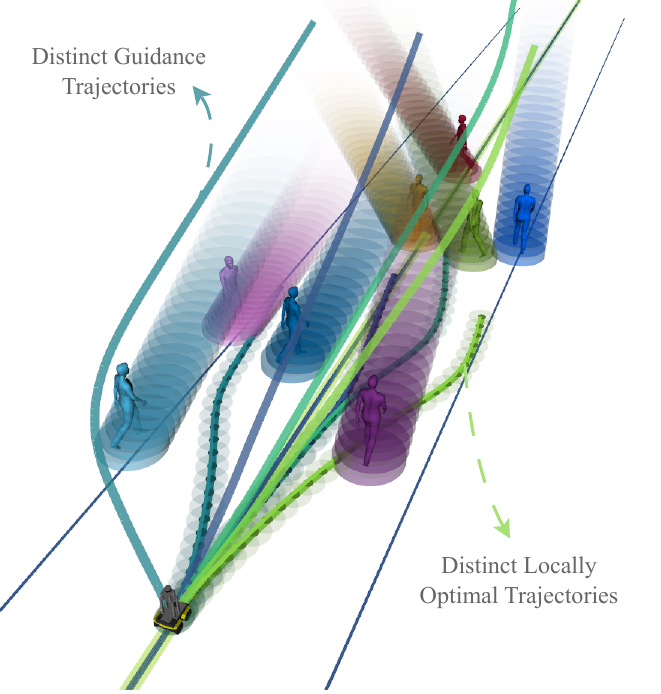}
    \caption{T-MPC first computes distinct guidance trajectories in the state space (time is visualized in the upwards direction). Each guidance trajectory initializes a local planner, resulting in several distinct locally optimized trajectories. The locally optimized trajectories each pass the obstacles (predicted future motion visualized as cylinders) in a distinct way.}
    \label{fig:eye-catcher}
\end{figure}%

Alternatively, \textit{global planners}, such as Randomly exploring Random Trees (RRT*)~\cite{karaman_sampling-based_2011} and motion primitives~\cite{werling_optimal_2010}, generate many feasible trajectories, evaluating safety and performance for each. Because they do not distinguish between high-level and low-level decisions, many redundant poor trajectories may be generated, leading to poor-quality motion plans under strict computational limits. This issue is especially prevalent in highly dynamic environments where trajectories need to be computed fast.

In this work, we present a planning framework, which we refer to as Topology-driven Model Predictive Control (T-MPC) (see Fig.~\ref{fig:eye-catcher}), that leverages the strengths of optimization-based planners and global planners. 

We present a global planner that identifies several distinct, high-level navigation options by considering the topology of the dynamic collision-free space. The underlying topology allows us to distinguish between similar and distinct trajectories. We then use each high-level trajectory as initialization for an optimization-based planner. The low-level planning problems are independent and are solved in parallel. Our framework does not modify the cost function of the optimization-based planner and can select the executed trajectory by comparing their optimal costs. 

\section{RELATED WORK AND CONTRIBUTION} \label{sec:related_work} %
Motion planning methods can be divided into local and global planning methods and combinations of these methods.
\subsubsection{Local Planners}
\textit{Local planners} such as Model Predictive Control (MPC)~\cite{brito_model_2019,ferranti_safevru_2019} typically formulate trajectory planning as a nonlinear optimization problem where performance (e.g., progress and smoothness) is optimized under constraints (e.g., dynamic constraints and collision avoidance). MPC can plan time efficient and smooth trajectories and handles various constraints, for example, to account for uncertainty in human behavior~\cite{zhu_chance-constrained_2019,wang_fast_2020,de_groot_scenario-based_2021, de_groot_scenario-based_2023}. However, because the collision-free space is nonconvex (obstacle regions are excluded), the optimized trajectory is locally optimal, and there may therefore exist a lower-cost motion plan than the returned solution. This may occasionally result in poor (e.g., slow, non-smooth) trajectories and can prevent MPC from returning a feasible trajectory in time.

To mitigate infeasibility with MPC, some authors propose to use two trajectories where one features as a contingency plan~\cite{pek_fail-safe_2021,alsterda_contingency_2019} improving planner safety. The planner may still perform poorly when the contingency plan is activated. Alternatively, robustness can be improved by running several optimizations in parallel. For example, in~\cite{adajania_multi-modal_2022}, an MPC is parallelized over goal locations, but requires a structured environment and a specific cost function and constraint set.

\subsubsection{Global Planners}
In contrast with local planners, \textit{global planners} do not rely on nonconvex optimization and therefore do not get trapped in local optima. Sampling-based global planners such as RRT*~\cite{karaman_sampling-based_2011} and Probabilistic Roadmaps (PRM)~\cite{kavraki_probabilistic_1996} plan by randomly sampling and connecting states in the configuration space until a goal configuration is reached. 
These methods typically consider static obstacles. In dynamic environments, RRT$^{\textrm{x}}$~\cite{otte_rrtx_2016} continuously rewires the graph. Recent work~\cite{orthey_multilevel_2024} greatly improved the computational efficiency of sampling-based planners for high-dimensional problems by using topological abstraction over fiber bundles. Unfortunately, these methods remain computationally demanding when dynamic constraints and collision avoidance are imposed on the problem and may return non-smooth trajectories. 

Motion primitive planners (e.g.,~\cite{werling_optimal_2010,stahl_multilayer_2019}) instead generate a large number of trajectories that are dynamically feasible by construction. The best trajectory is identified by scoring each trajectory. Motion primitives planners efficiently compute smooth trajectories, but discretize the possible maneuvers which can lead to infeasibility and inefficient robot motion. For static environments,~\cite{ortiz-haro_idb-_2023} presents a smooth global planner that repairs dynamic mismatches between global plans through trajectory optimization. It is however computationally demanding. Similarly, PiP-X~\cite{m_jaffar_pip-x_2023} combines graph-search with funnels to find robust dynamically feasible paths but may return inefficient trajectories. In~\cite{marcucci_motion_2022}, the motion planning problem with collision avoidance is solved via a convex optimization by utilizing Graphs of Convex Sets (GCS). This approach is promising, but is not real-time yet and imposes limitations on the trajectory end point, supported dynamics and constraints.

\subsubsection{Guidance Planners}
Local planners typically receive an initial trajectory or reference path from a global planner. This global planner takes into account static obstacles and the overall route to the goal, which helps prevent the local planner from encountering deadlocks~\cite{ziegler_making_2014,zheng_ros_2021}. The performance can be further enhanced by incorporating dynamic obstacles into the global planner. We refer to global planners that consider dynamic obstacles as `guidance planners'. 
For example, for self-driving vehicle applications, \cite{eiras_two-stage_2022} initializes an MPC in the desired behavior with a simplified Mixed-Integer Linear Program (MILP). In~\cite{ding_epsilon_2022}, a behavior planner based on a Partially-Observable Markov Decision Process (POMDP) guides a local motion planner in interactive scenarios for a self-driving vehicle. Both methods rely on a structured environment.%

To compute a suitable initial guess for a local planner considering obstacles, several authors~\cite{park_homotopy-based_2015, bhattacharya_topological_2012, rosmann_integrated_2017, altche_partitioning_2017, yi_model_2019,de_groot_globally_2023,zhou_robust_2020, kretzschmar_socially_2016, mavrogiannis_winding_2023} have noted that local optima related to collision avoidance link to the topology of trajectories through the collision-free space. Roughly speaking, two trajectories are in the same \textit{homotopy} class if they can be smoothly transformed into each other in the collision-free space~\cite{bhattacharya_topological_2012} (e.g., when they evade the obstacles on the same side). Unfortunately, homotopy classes of trajectories are difficult to compute in general. If the environment can be consistently represented as a graph, then homotopy classes of trajectories can be identified through distinct paths over the graph~\cite{park_homotopy-based_2015}. This applies, for example, in structured autonomous driving applications through the lane structure of the road network~\mbox{\cite{altche_partitioning_2017, yi_model_2019}}.

Without structure in the environment, it is difficult to compute a trajectory in each homotopy class. 
Graphs can be constructed from static obstacles. In~\cite{cao_dynamic_2019}, Delauney triangulation is used to identify passable gaps between dynamic obstacles in a global planner. Voronoi graphs are used in~\cite{kretzschmar_socially_2016} to identify homotopy classes with respect to static obstacles and in~\cite{rosmann_integrated_2017} include each dynamic obstacle and their predicted motion as a static obstacle. 
Trajectories are generated from the homotopy class description in~\cite{mavrogiannis_hamiltonian_2021} by modeling interactions as a physical vortex system. These graph-based and generative approaches can exhaust the possible homotopy classes, but scale poorly to crowded environments (as noted in~\cite{rosmann_integrated_2017} and~\cite{mavrogiannis_hamiltonian_2021}).

Instead, several works, such as~\cite{zhou_robust_2020, bhattacharya_topological_2012,rosmann_integrated_2017}, compute distinct trajectories by filtering out homotopy equivalent trajectories during planning. For \mbox{$3$-D} navigation among static obstacles,~\cite{zhou_robust_2020} introduces Universal Visibility Deformation (UVD) to compare trajectories. Trajectories are UVD equivalent if they can be connected without collision at several intermediate times. The authors present a visibility-PRM~\cite{simeon_visibility-based_2000} to generate UVD-distinct trajectories. 
In $2$D dynamic environments, homotopy classes are typically compared via winding numbers~\cite{berger_topological_2001} or the H-signature~\cite{bhattacharya_topological_2012}. \textit{Winding numbers} track the relative angle between the robot and dynamic obstacles over their trajectories. They were used in~\cite{kretzschmar_socially_2016} to distinguish homotopy classes of trajectories with respect to dynamic obstacles. In~\cite{mavrogiannis_winding_2023}, an MPC with winding numbers in the cost function was proposed to motivate passing progress. The \textit{H-signature} uses \textit{homology} classes as an approximation for homotopy classes. The work in~\cite{rosmann_integrated_2017}, which relates most closely to this work, applied this approximation for $2$-D navigation among static obstacles. Their planner, Time Elastic Band (TEB) Local Planner, identifies several trajectories in distinct homology classes (using regular PRM) and uses each to initialize a soft-constrained optimization-based planner. TEB has, however, three main limitations that can hinder its performance in dynamic environments. Firstly, the trajectory topology is confined to the static workspace, treating dynamic obstacles and their future motion as static obstacles. Secondly, the guidance planner is designed to reach a single goal. Lastly, the local planner lacks hard constraints.\\ 

In this work, we introduce a topology-guided planner that is different from these existing works in four ways. First, we consider homotopy classes in the \textit{dynamic} collision-free space, that includes time, to incorporate the motion of dynamic obstacles (contrary to~\cite{rosmann_integrated_2017,zhou_robust_2020}). Second, our framework does not modify the cost function (i.e., the performance criteria) of the local planner (contrary to~\cite{mavrogiannis_winding_2023,adajania_multi-modal_2022,marcucci_motion_2022}). Third, our planner does not rely on a structured environment (contrary to \cite{yi_model_2019,altche_partitioning_2017,ding_epsilon_2022}). Finally, our guidance planner can handle the case where its goal is blocked (contrary to~\cite{zhou_robust_2020,rosmann_integrated_2017}) by considering multiple goal positions. 

In addition, we \textit{enforce} the final trajectories to be in distinct homotopy classes using constraints in the local planner and we show that it is not sufficient to initialize the solver in a homotopy class (contrary to~\cite{rosmann_integrated_2017}). By consistently planning distinct trajectories, we can reidentify trajectories of prior planning iterations and use this information to make the planner more consistent and decisive. Our method furthermore supports the H-signature, winding numbers and UVD for comparing homotopy classes.

\subsection{Contribution}
In summary, our topology-driven parallel planning framework, T-MPC, contributes to the state of the art as follows:
\begin{enumerate}
    \item A planning framework for dynamic environments that optimizes trajectories in multiple distinct homotopy classes in parallel. Our framework extends existing optimization-based local planners, improving their time efficiency, safety and consistency.
    \item A fast guidance planner that computes homotopy distinct trajectories through the dynamic collision-free space towards multiple goal positions.%
\end{enumerate}
We validate our proposed framework in simulation on a mobile robot navigating among interactive pedestrians. 
We show how our framework can accommodate different trajectory optimization approaches (e.g.,~\cite{brito_model_2019} in the nominal case, and~\cite{zhu_chance-constrained_2019} to accommodate Gaussian uncertainties added to the motion of the dynamic obstacles). We show how our framework enhances the performance of~\cite{brito_model_2019,zhu_chance-constrained_2019} out of the box and we compare against three additional baselines (\cite{werling_optimal_2010,rosmann_integrated_2017,de_groot_globally_2023}). We finally demonstrate our planner in the real world on a mobile robot navigating among five pedestrians. Our C\texttt{++}/ROS implementation of T-MPC will be released open source.

This work is an extension of our earlier conference publication~\cite{de_groot_globally_2023}. In~\cite{de_groot_globally_2023}, we computed a single guidance trajectory and followed it with a local planner by adding a tracking term. Compared to~\cite{de_groot_globally_2023}, we compute and optimize multiple distinct guidance trajectories in parallel. In addition, we derive constraints from the guidance trajectory such that the cost of the optimization is unmodified and can be used to compare optimized trajectories. Finally, we improved the robustness and consistency of the guidance planner and extended the experimental evaluation. 

The rest of this work is organized as follows. We introduce the planning problem in Sec.~\ref{sec:problem_formulation}. The planning framework is described and analyzed in Sec.~\ref{sec:method}. Simulation and real-world results are presented in Sec.~\ref{sec:simulation} and Sec.~\ref{sec:experiments}, respectively, followed by a discussion in Sec.~\ref{sec:discussion}.

\section{PROBLEM FORMULATION}\label{sec:problem_formulation}
We consider discrete-time nonlinear robot dynamics
\begin{equation}
    \b{x}_{k + 1} = f(\b{x}_k, \b{u}_k),
\end{equation}
where $\b{x}_k\in\mathbb{R}^{n_x}$ and $\b{u}_k\in\mathbb{R}^{n_u}$ are the state and input at discrete time instance $k$, $n_x$ and $n_u$ are the state and input dimensions respectively and the state contains the 2-D position of the robot $\b{p}_k = (x_k, y_k)\in\mathbb{R}^2 \subseteq{\mathbb{R}^{n_x}}$. 

The robot must avoid moving obstacles in the environment. The position of obstacle $j$ at time $k=0$ is denoted $\b{o}^j_0\in\mathbb{R}^2$ and we assume that for each obstacle, predictions of its positions over the next $N$ time steps are provided to the planner (i.e., $\b{o}^j_1, \hdots, \b{o}^j_{N}$) at each time instance. The collision region of the robot is modeled by a disc of radius $r$ and that of each obstacle $j$ by a disc with radius $r^j$ (see Fig.~\ref{fig:homotopy2d}).

For high-level planning with dynamic collision avoidance, we consider the simplified state space  
$\mathcal{X} \coloneqq \mathbb{R}^2 \times [0, T]$, with $[0, T]$ a continuous finite time domain (see Fig.~\ref{fig:homotopy3d}). The area of the workspace occupied by the union of obstacles at time $t$ is denoted by $\mathcal{O}_t \subset \mathbb{R}^2$ and the obstacle set in the state space is thus $\mathcal{O} \coloneqq \bigcup_{\forall t \in [0, T]}(\mathcal{O}_t, t) \subset \mathcal{X}$. The collision free state space (or \textit{free space}) is denoted $\mathcal{C} \coloneqq \mathcal{X} \textrm{\textbackslash} \mathcal{O}$. A trajectory is a continuous path through the state space, $\b{\tau} : [0, 1] \to \mathcal{X}$. The goal of the robot is to traverse along a given reference path $\b{\gamma} : [0, 1] \to \mathbb{R}^2$ without colliding with the obstacles while tracking a reference velocity $v_{\textrm{ref}}$. It is allowed to deviate from the path.

\begin{figure}[t]
    \centering
    \begin{subfigure}[t]{0.14\textwidth}
        \centering
        \includegraphics[width=\textwidth]{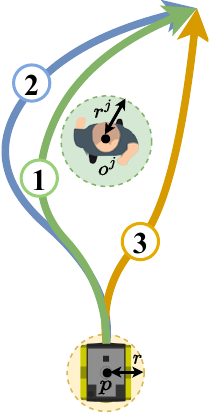}%
        \caption{Planning scene.}%
        \label{fig:homotopy2d}
    \end{subfigure}
    \hspace{0.05\textwidth}%
    \begin{subfigure}[t]{0.22\textwidth}
        \centering
        \includegraphics[width=\textwidth]{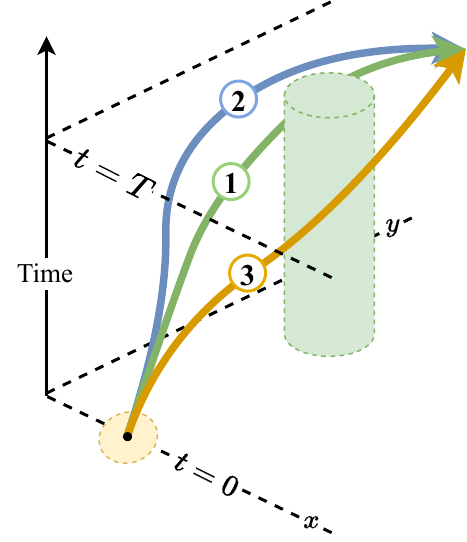}%
        \caption{State space view.}%
        \label{fig:homotopy3d}
    \end{subfigure}
    \caption{(a) Depiction of the planning problem and (b) equivalent in the state-space. Trajectory $1$ and $2$ are in the same homotopy class while trajectory $1$ and $3$ are in distinct homotopy classes.}
    \label{fig:homotopy}%
\end{figure}
\begin{figure*}
    \centering
    \includegraphics[width=\textwidth,trim={0cm 0cm 2cm 0cm},clip]{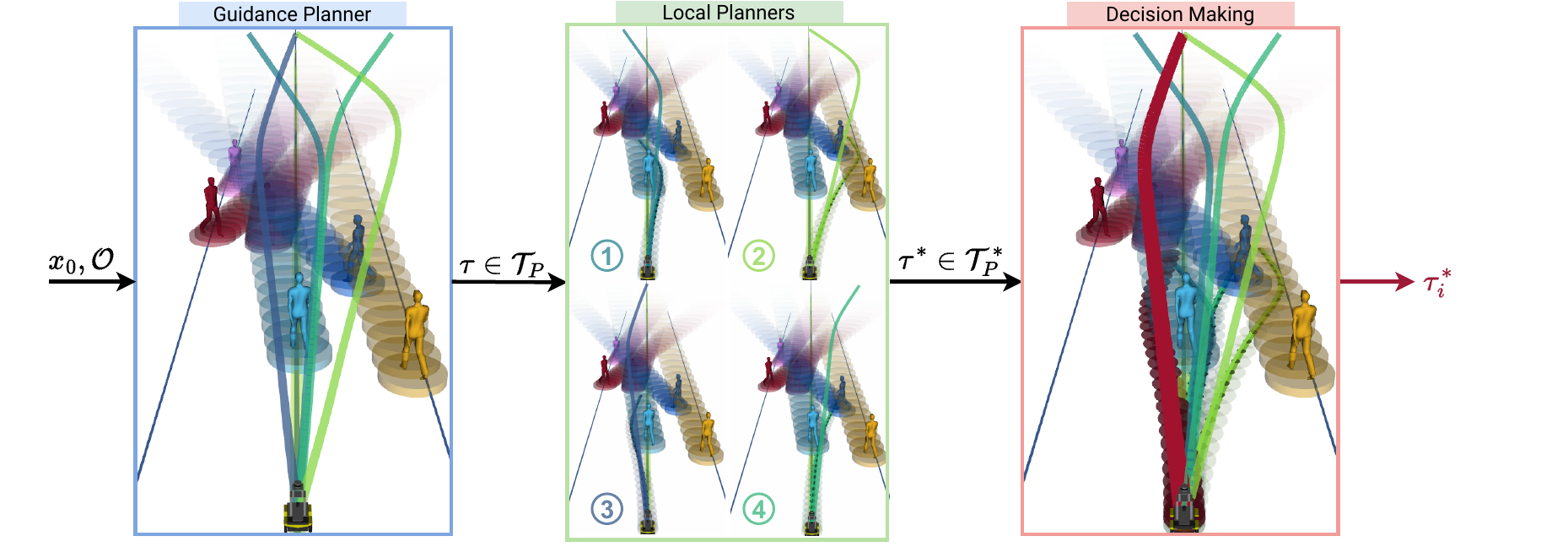}
    \caption{Schematic of T-MPC. An environment with several obstacles and a robot is visualized in $x, y, t$ (time in the upwards axis). Obstacle motion predictions are denoted with cylinders. {{\color{colorfor1} \textbf{(1)}}} A guidance planner (Sec.~\ref{sec:global_planner}) finds $P=4$ trajectories (visualized with colored lines) from the robot initial position to one of the goals. Each of these trajectories is in a distinct homotopy class in the state space. {{\color{colorfor2} \textbf{(2)}}} Each trajectory guides a local planner (Sec.~\ref{sec:local_planner}) as initial guess and through a set of homotopy constraints. Four guidance trajectories and optimized trajectories (as occupied regions for each step) are visualized. {{\color{colorfor3} \textbf{(3)}}} The optimized trajectories are compared through their objective value (Sec.~\ref{sec:decision_making}) and a single trajectory (in red) is excuted by the robot.}
    \label{fig:method-schematic}%
\end{figure*}

\subsection{Optimization Problem}
We formalize the planning problem as the following trajectory optimization problem over a horizon of $N$ steps
\begin{subequations}
    \label{eq:local_optimization}
\begin{align}
    \min_{\b{u}\in \mathbb{U}, \b{x}\in\mathbb{X}} \quad & \sum_{k = 0}^{N} J(\b{x}_k, \b{u}_k) \label{eq:local_objective}\\
    \textrm{s.t.} \quad \quad & \b{x}_{k + 1} = f(\b{x}_k, \b{u}_k), \: \forall k \label{eq:local_dynamics}\\
    &\b{x}_{0} = \b{x}_{\textrm{init}}\label{eq:init}\\
    &g(\b{x}_k, \b{o}^j_k) \leq 0, \: \forall k, j,\label{eq:local_constraints}
\end{align}
\end{subequations}
where the cost function $J$ in \eqref{eq:local_objective} expresses the planning objectives (e.g., following reference path $\b{\gamma}$). Robot dynamics and initial conditions are imposed by \eqref{eq:local_dynamics} and \eqref{eq:init}, respectively and collision avoidance constraints are imposed by \eqref{eq:local_constraints}. %

Because dynamic obstacles puncture holes in the free space, the free space associated with the constraints~\eqref{eq:local_constraints} is nonconvex. %
Nonlinear optimization algorithms, solving this problem, return just one of possibly many local optimal trajectories. The initial guess provided to them determines which local optimal trajectory is returned.
It is generally unclear how close this trajectory is to the globally optimal trajectory (i.e., the best trajectory under the specified cost). In this work, we want to leverage this weakness to explore in parallel multiple locally optimal trajectories (provided as initial guesses on $\b{x}$) that evade obstacles in a distinct way.

\subsection{Homotopic Trajectories}\label{sec:preliminary}

To achieve the goal above, we rely on the concept of homotopic trajectories, which can be formalized as follows:
\begin{definition}\label{def:homotopic_trajectories}
    \cite{bhattacharya_topological_2012}~(Homotopic Trajectories) Two paths connecting the same start and end points $\b{x}_s$ and $\b{x}_g$ respectively, are homotopic if they can be continuously deformed into each other without intersecting any obstacle. Formally, if $\b{\tau}_1, \b{\tau}_2 \in T$ represent two trajectories, with $\b{\tau}_1(0) = \b{\tau}_2(0) = \b{x}_s$ and $\b{\tau}_1(1) = \b{\tau}_2(1) = \b{x}_g$, then $\b{\tau}_1$ is homotopic to $\b{\tau}_2$ iff there exists a continuous map $\eta : [0, 1] \times [0, 1] \to \mathcal{C}$ such that $\eta(\alpha, 0) = \tau_1(\alpha) \forall \alpha \in [0, 1], \ \eta(\beta, 1) = \tau_2(\beta), \forall \beta \in [0, 1]$ and $\eta(0, \gamma) = \b{x}_s, \ \eta(1, \gamma) = \b{x}_{g} \forall \gamma \in [0, 1]$.
\end{definition}
If two trajectories are homotopic, they are said to be in the same homotopy class. An example is depicted in Fig.~\ref{fig:homotopy}. To distinguish between trajectories in different homotopy classes, we make use of the homotopy comparison function 
\begin{equation}\label{eq:homotopy_comparison}
    \mathcal{H}(\b{\tau}_i, \b{\tau}_j,\mathcal{O})\!=\! \begin{cases}
        1, & \b{\tau}_i, \b{\tau}_j \textrm{ in the same homotopy class}\\
        0, & \textrm{otherwise}
    \end{cases}%
\end{equation}
Verifying whether two trajectories are in the same homotopy class can be computationally inefficient. %
We support the H-signature~\cite{bhattacharya_topological_2012}, winding numbers~\cite{berger_topological_2001} and UVD~\cite{zhou_robust_2020} that allow us to approximately perform this verification in real-time. Details of the three methods are provided in Appendix~\ref{ap:homotopy_comparison}.%

\section{TOPOLOGY-DRIVEN MODEL PREDICTIVE CONTROL}\label{sec:method}
In this section, we propose T-MPC, a topology-guided planner that optimizes trajectories in multiple distinct homotopy classes in parallel.

Our planner consists of two %
components: a high-level guidance planner and multiple identical low-level local planners (see Fig.~\ref{fig:method-schematic}). The \textit{guidance planner} $G$ generates homotopy distinct trajectories through the free space%
\begin{equation}
    G(\b{x}_0, \mathcal{P}_g, \mathcal{C}) = \{\b{\tau}_1, \hdots, \b{\tau}_P\} \eqqcolon \mathcal{T}_P,\label{eq:global-planner}%
\end{equation}
where $\b{x}_0$ denotes the robot initial state and $\mathcal{P}_g$ denotes a set of goal positions. Each \textit{local planner} is initialized with one of the guidance trajectories and optimizes the trajectory in the same homotopy class. With $N$ the horizon of the guidance and local planners\footnote{The guidance planner horizon could extend beyond the horizon of the local planner. We set them equal for simplicity.}, each local planner defines a mapping $L : \mathcal{X}^{N} \to \mathcal{X}^{N}$,
\begin{equation}
    L(\b{\tau}_i) = \b{\tau}_i^*.\label{eq:local-planner}
\end{equation}%
To ensure that the local planner optimizes in the provided homotopy class, we append a set of constraints derived from the guidance trajectory. 
These constraints are appended to existing collision avoidance constraints to adapt the planner to the globalized framework. %
The proposed planner computes locally optimal trajectories $\mathcal{T}^*_P \coloneqq \{\b{\tau}^*_1, \hdots, \b{\tau}^*_P\}$ in several distinct homotopy classes.

\subsection{Guidance Planner - Overview}\label{sec:global_planner}%
The goal of the guidance planner is to quickly compute several homotopy distinct trajectories through the free space. Similarly to~\cite{zhou_robust_2020,rosmann_integrated_2017}, we perform this search via Visibility-Probabilistic RoadMaps (Visibility-PRM~\cite{simeon_visibility-based_2000}), a sampling-based global planner. The modifications that we make ensure that the graph remains consistent over successive iterations. 

\begin{figure}[t]
    \centering
    \begin{subfigure}{0.25\textwidth}
        \centering
        \includegraphics[width=\textwidth,trim={0cm 0cm 0cm 0cm},clip]{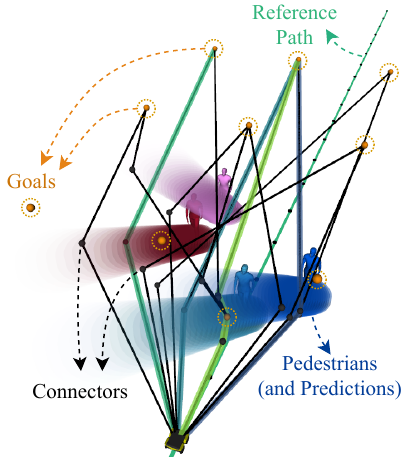}
        \caption{}%
        \label{fig:prm-graph}
    \end{subfigure}
    \begin{subfigure}{0.18\textwidth}
        \centering
        \includegraphics[width=\textwidth,trim={2cm 0cm 3cm 0cm},clip]{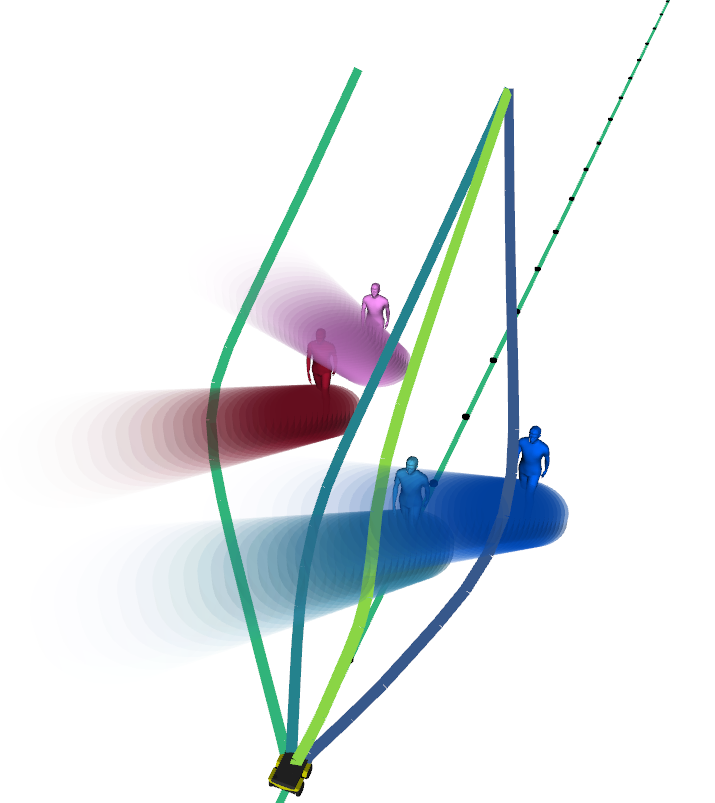}
        \caption{}%
        \label{fig:prm-trajectories}
    \end{subfigure}
    \caption{Illustration of the guidance planner in the state-space (time in the upwards axis). Visualization follows Fig.~\ref{fig:method-schematic}. (a) The visibility-PRM graph (black lines and dots) explores the free space toward the goals placed at $t=T$ around the reference path (orange dots). The homotopy distinct guidance paths (colored lines) are obtained by searching the graph. (b) The final trajectories are smoothened.}
    \label{fig:guidance-planner}%
\end{figure}%

The guidance planner is outlined in Algorithm~\ref{alg:guidance_planner} and visualized in Fig.~\ref{fig:guidance-planner}. Details of the algorithm are given in Sec.~\ref{sec:detailed_guidance_planner}. We give a high-level overview here. First, \textbf{Visibility-PRM} constructs a sparse graph through the state space from the robot position to a set of goals, where each connection is homotopy distinct (line $1$). The goals represent end points for the guidance planner and are placed along the reference path. For each goal, \textbf{DepthFirstSearch} (line $2$) searches in this graph for the shortest $P$ trajectories that reach it. Any homotopy equivalent trajectories are filtered out by \textbf{FilterAndSelect} (line $3$), ensuring that the remaining trajectories are in distinct homotopy classes. The $P$ trajectories that seem most promising are selected by a heuristic that prefers its goal to be as close as possible to the reference path at the reference velocity (see Fig.~\ref{fig:prm-graph}). 
\textbf{IdentifyAndPropagate} (line~$4$) verifies if any of the selected trajectories are equivalent to trajectories of the previous planning iteration. This reidentification makes it possible to follow the same passing behavior over multiple planning iterations. We finally propagate the nodes in the Visibility-PRM graph by lowering their time state by the planning time step.

Through this process, we obtain in each iteration $P$ piecewise linear trajectories $\mathcal{T}_P = \{\b{\tau}_1, \hdots, \b{\tau}_P\}$ that each connects the robot position to one of the goals (see Fig.~\ref{fig:method-schematic}). We finally smoothen these trajectories and fit cubic splines to make them differentiable (see Fig.~\ref{fig:prm-trajectories}). More details can be found in~\cite{de_groot_globally_2023}. The smoothening procedure produces only a small displacement in trajectories to maintain their homotopy class.

These trajectories serve as initializations for the local planners, described in Sec.~\ref{sec:local_planner}. %

\begin{algorithm}[t]
    \caption{Guidance Planner}
    \label{alg:guidance_planner}
    \KwIn{$\mathcal{C}$, $\b{x}_0$, $\b{x}_N$, previous graph $\mathcal{G}^-$, previous trajectories $\mathcal{T}_P^-$}

    $\mathcal{G} \ \leftarrow$ \textbf{Visibility-PRM}$(\mathcal{C}, \b{x}_0, \b{x}_N, \mathcal{G}^-)$\\
    $\{\b{\tau}_0, \hdots, \b{\tau}_{N_{GS}} \} \ \leftarrow$ \textbf{DepthFirstSearch}$(\mathcal{G})$\\
    $\mathcal{T}_P = \{\b{\tau}_0, \hdots, \b{\tau}_{P} \} \ \leftarrow$ \textbf{FilterAndSelect}$(\{\b{\tau}_0, \hdots, \b{\tau}_{N_{GS}} \})$\\

    $\mathcal{G}^- \ \leftarrow$ \textbf{IdentifyAndPropagate}$(\{\b{\tau}_0, \hdots, \b{\tau}_{P} \}, \mathcal{T}_P^-)$\\

    \KwOut{$\mathcal{T}_P$}
\end{algorithm}

\subsection{Guidance Planner - Detailed Description}\label{sec:detailed_guidance_planner} %
We detail each step of Algorithm~\ref{alg:guidance_planner} in the following.

\textbf{Visibility-PRM} computes sparse paths through the free space by randomly sampling positions and creating either a \textit{Guard} or \textit{Connector} node at the sampled position. The type of node depends on the number of Guards that it can directly connect to without colliding (i.e., which Guards are \textit{visible}). A \textit{Guard} is added if no other Guards are visible. A \textit{Connector} (see black dots in Fig.~\ref{fig:guidance-planner}) is added when exactly two Guards are visible and its connection to the Guards is feasible (e.g., satisfying velocity and acceleration limits). Similar to~\cite{zhou_robust_2020}, we also check if any Connectors link to the same Guards (referred to as \textit{neighbors}). If there are neighbors, we keep the new connection if it is distinct from existing connections, which we verify with homotopy comparison function~\eqref{eq:homotopy_comparison}. If it is equivalent and more efficient than the existing connection (e.g., if its connection is shorter), then we replace the existing connector with the new connector. In regular visibility-PRM~\cite{simeon_visibility-based_2000}, the graph is initialized with a Guard at the start and goal positions and new nodes are drawn up to a time or node limit. More details of the algorithm can be found in~\cite[Algorithm 1]{de_groot_globally_2023}.

\subsubsection*{Multiple Goals in Visibility-PRM} 
In this work, we address the limitation that a single goal must be reached by Visibility-PRM, which causes the planner to fail when that goal cannot be reached. We propose to add a Goal node type to Visibility-PRM. Goals inherit the properties of Guards but are inserted initially and are likely visible to each other. When a Connector can connect to multiple Goals, we single out the Goal with the lowest distance to the point on the reference path reached with the reference velocity (i.e., our ideal goal). By supporting multiple goals, we increase the robustness of the guidance planner. In practice, we deploy a grid of goals centered around the reference path (see Fig.~\ref{fig:guidance-planner}).

\subsubsection*{Homotopy Comparison}
We use the homotopy comparison function \eqref{eq:homotopy_comparison} to verify if two trajectories are in the same homotopy class. We implemented \eqref{eq:homotopy_comparison} with the H-signature~\cite{bhattacharya_topological_2012}, winding numbers~\cite{berger_topological_2001} and UVD~\cite{zhou_robust_2020}. Appendix~\ref{ap:homotopy_comparison} provides details on these methods. In our experiments, we use the H-signature that joins the two trajectories to be compared into a loop and verifies if that loop encircles any moving obstacles. If it does, then the two trajectories pass obstacles differently and belong to different homotopy classes.

\textbf{DepthFirstSearch} searches for $P$ paths to each goal, with each search implemented similar to~\cite[Algorithm 1]{rosmann_integrated_2017}.

\textbf{FilterAndSelect} uses homotopy comparison function \eqref{eq:homotopy_comparison} to remove equivalent trajectories to different Goals found by DepthFirstSearch. The set of filtered trajectories $\mathcal{T}_F$ therefore satisfy
\begin{equation}
    \mathcal{H}(\b{\tau}_i, \b{\tau}_j, \mathcal{O}) = 0, \ \forall i, j, i\neq j, \ \b{\tau}_i, \b{\tau}_j \in \mathcal{T}_F.
\end{equation}
The $P$ lowest cost trajectories in $\mathcal{T}_F$ constitute the output $\mathcal{T}_P$.

\textbf{IdentifyAndPropagate} uses homotopy comparison function \eqref{eq:homotopy_comparison} to link new trajectories to trajectories found in the previous iteration. It checks for each previous trajectory $\b{\tau}_i^- \in \mathcal{T}^-_P$ if
\begin{equation}
    \exists \b{\tau}_j \in \mathcal{T}_P, \mathcal{H}(\b{\tau}_i^-, \b{\tau}_j) = 1.
\end{equation}%
A unique identifier, tied to the homotopy class, is passed from $\b{\tau}_i^-$ to $\b{\tau}_j$ if the latter exists. We can use this identifier to decide which trajectory to follow (see Sec.~\ref{sec:decision_making}).

\subsection{Local Planner}\label{sec:local_planner}
To refine the trajectories of the guidance planner, we apply $P$ local planners in parallel. Each local planner refines one of the guidance trajectories $\b{\tau}_i$ and needs to ensure that the final trajectory is dynamically feasible and that it satisfies any other imposed constraints. We pose the following definition:
\begin{definition}\label{def:local_planner}
    (Local Planner) The local planner is an algorithm $L : \mathcal{X}^{N} \to \mathcal{X}^{N}$ that respects constraints. %
\end{definition}
This definition captures many existing optimization-based planners. In this work, we define the local planner through the trajectory optimization in Eq.~\eqref{eq:local_optimization}, where we make two modifications to ensure that the optimized trajectories are in the homotopy class of the associated guidance trajectory. 
First, the trajectory optimization of each local planner uses its guidance trajectory as the initial guess for $\b{x}$. 
The initial guess speeds up convergence but does not guarantee that the optimized trajectory remains in the same homotopy class when there are obstacles. In the next section (Sec.~\ref{sec:constraint-validation}), we provide an example where initialization in distinct homotopy classes still leads to identical optimized trajectories.

\begin{figure}[t]
    \centering
    \begin{subfigure}{0.4\textwidth}
        \centering
        \includegraphics[width=\textwidth]{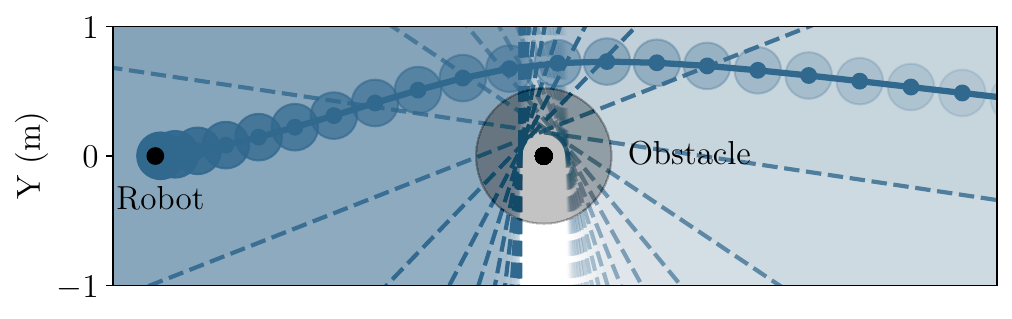}
        \caption{Local planner $1$ plans to evade the obstacle left.}%
        \label{fig:planner1}
    \end{subfigure}%

    \begin{subfigure}{0.4\textwidth}
        \centering
        \includegraphics[width=\textwidth]{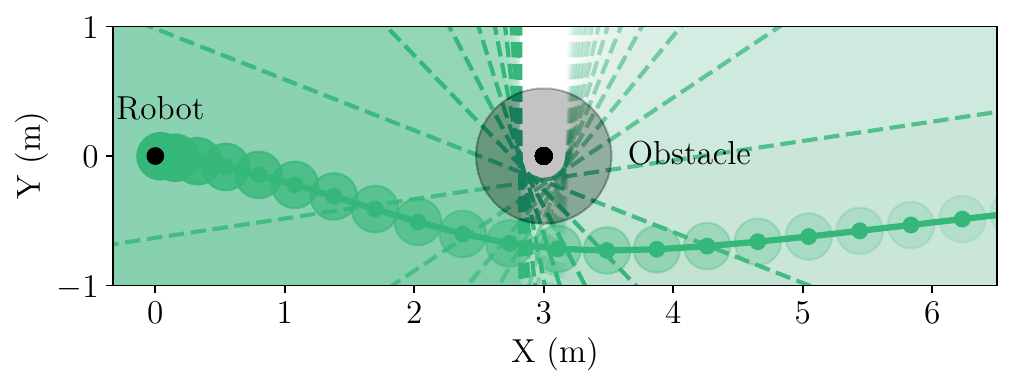}
        \caption{Local planner $2$ plans to evade the obstacle right.}%
        \label{fig:planner2}
    \end{subfigure}
     \caption{Two distinct locally planned trajectories for a robot (black dot) evading an obstacle (black region and dot) that is \textit{static} ($\b{o}_k = \b{o}_0, \ k = 1, \hdots, N$). For both planners, we depict the topology constraints for each time step in their respective colors showing the constraint boundaries (broken lines) and their feasible region (colored regions with increasing transparency over time).}     
    \label{fig:guidance-polygons}
\end{figure}

\begin{figure*}
    \centering%
    \begin{subfigure}{0.32\textwidth}%
        \centering%
        \includegraphics[width=\textwidth,trim={0pt 0cm 0pt 0pt},clip]{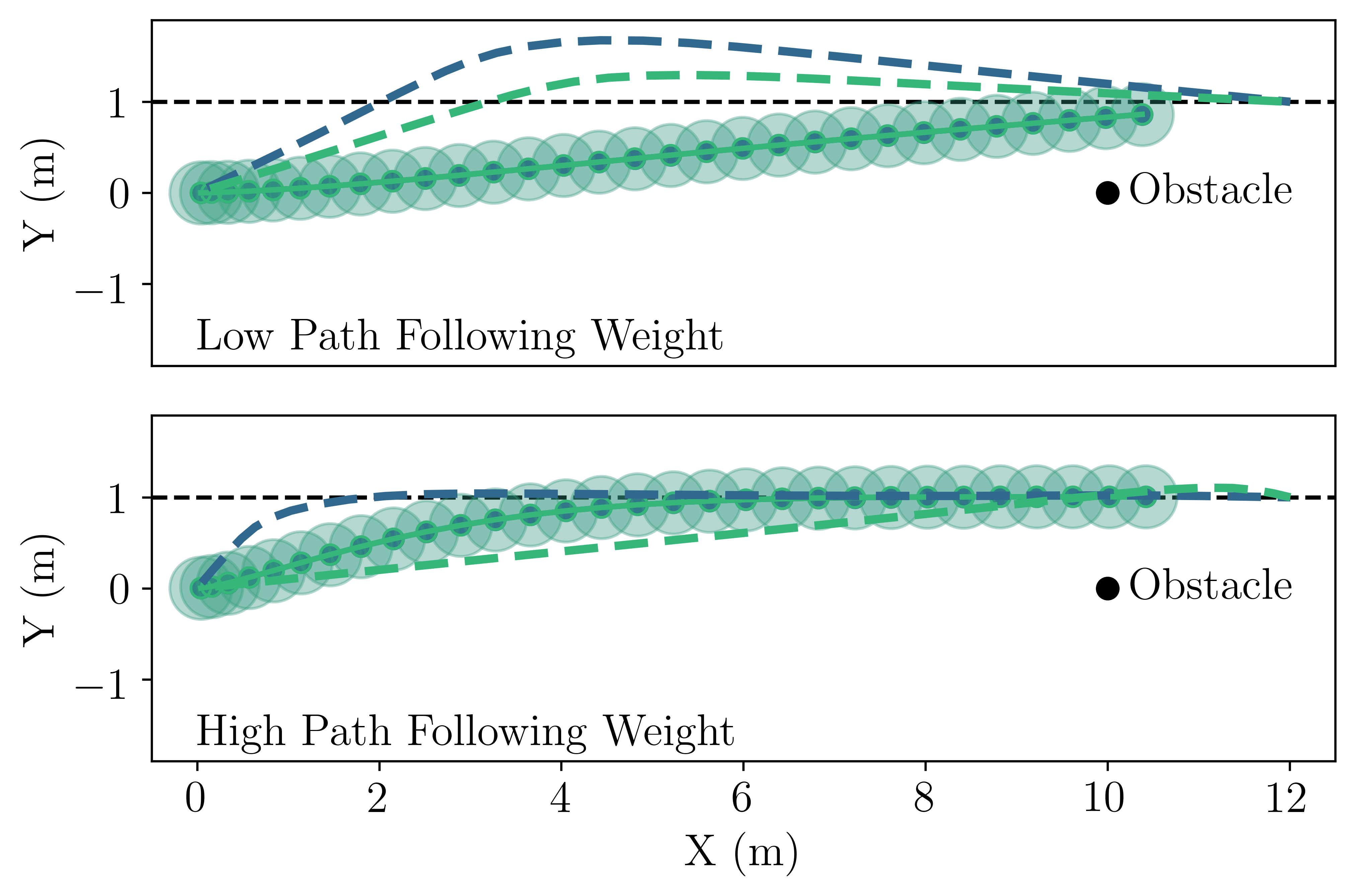}
        \caption{Without homotopy comparison~\eqref{eq:homotopy_comparison}.}%
        \label{fig:no-homotopy}
    \end{subfigure}%
    \begin{subfigure}{0.32\textwidth}%
        \includegraphics[width=\textwidth,trim={0pt 0cm 0pt 0pt},clip]{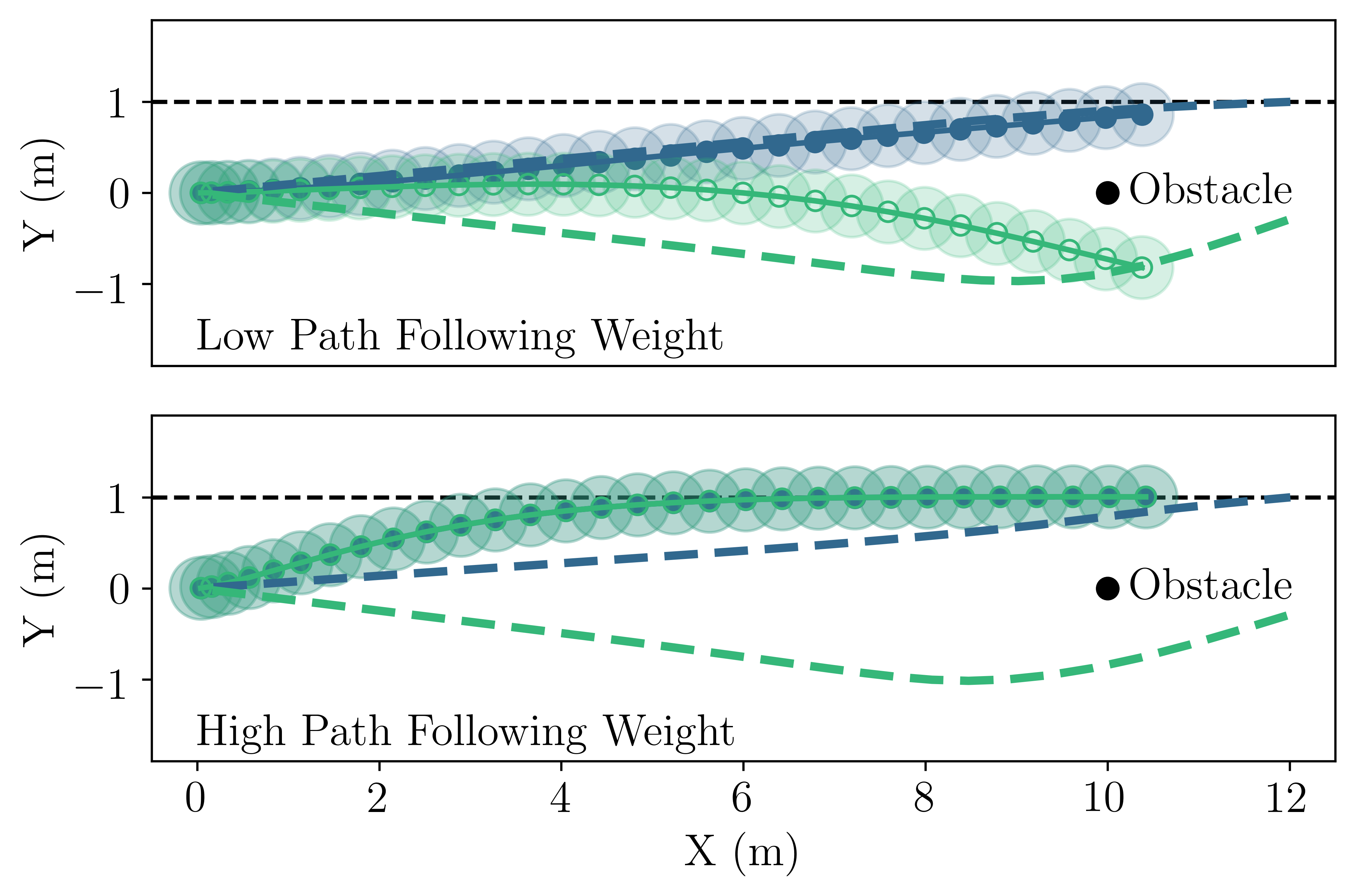}
        \caption{Without constraints~\eqref{eq:homotopy_constraints_opt}.}
        \label{fig:no-constraints}
    \end{subfigure}%
    \begin{subfigure}{0.32\textwidth}%
        \includegraphics[width=\textwidth,trim={0pt 0cm 0pt 0pt},clip]{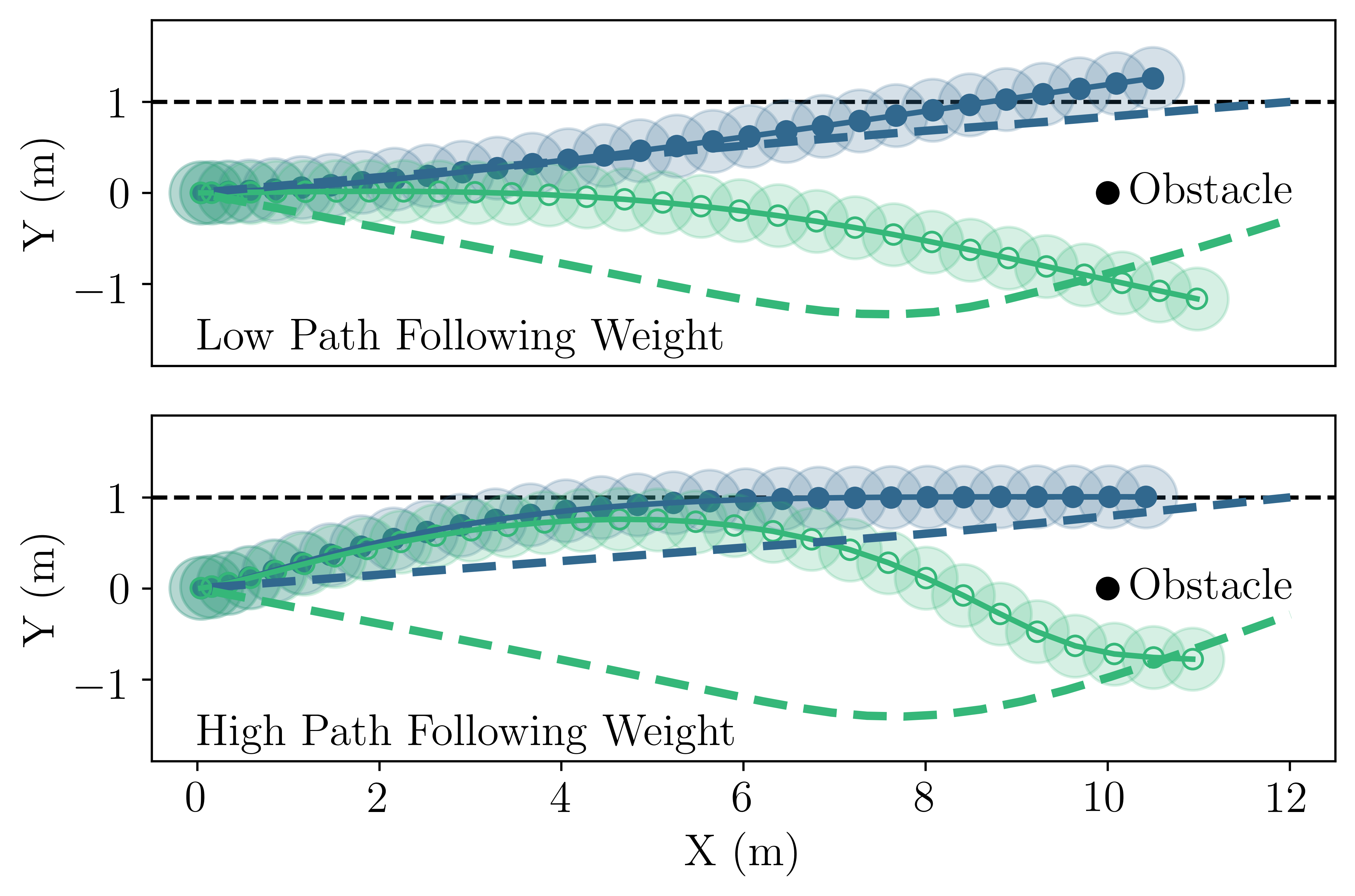}
        \caption{Our proposed method.}
        \label{fig:no-ablation}
    \end{subfigure}%
    \caption{Planned trajectories (lines with shaded discs) tracking a reference path (dashed black line) while avoiding a static obstacle with two guidance trajectories (dashed lines) for a low ($0.01$) and high ($0.3$) path following weight. (a) Without homotopy comparison~\eqref{eq:homotopy_comparison}, guidance trajectories are not distinct and optimized trajectories are identical. (b) Without homotopy constraints, increasing the path following weight results in identical trajectories. (c) With homotopy comparison~\eqref{eq:homotopy_comparison} and homotopy constraints~\eqref{eq:homotopy_constraints_opt}, optimized trajectories are distinct.}
    \label{fig:constraint-ablation}
\end{figure*}

To ensure that the homotopy class of the guidance trajectory is respected, we add to each local planner a set of constraints $g_H(\b{x}_k, \b{o}^j_k, \b{\tau}_{i, k})$. For this purpose, we construct for each time instance $k$ and obstacle $j$ a linear constraint between the guidance trajectory and obstacle position (see Fig.~\ref{fig:guidance-polygons}). With guidance trajectory $\b{\tau}_i$ and obstacle trajectory $\b{o}$, these constraints are given by $\b{A}_k\b{x}_k \leq b_k$, where
\begin{equation}
    \b{A}_k = \frac{\b{o}_k - \b{\tau}_{i, k}}{||\b{o}_k - \b{\tau}_{i, k}||},\quad b_k = \b{A}_k^T(\b{o}_k - A_k(\beta(r + r_{\textrm{obs}}))). \label{eq:homotopy-constraints}
\end{equation}
The relaxation factor $0 \leq \beta \leq 1$ scales the distance that the constraints enforce from each obstacle. A key observation is that, with other collision avoidance constraints in place, the topology constraints can be relaxed ($\beta \approx 0$) such that they are inactive at the obstacle boundary. Since the constraints do ensure that the trajectory remains on the same side of each obstacle, the optimized trajectory is %
in the same homotopy class as the initialization provided by the guidance planner.

The resulting homotopy preserving local planner is given by
\begin{subequations}
    \label{eq:guided_local_optimization}
    \begin{align}
        J_i^* = \min_{\b{u}\in \mathbb{U}, \b{x}\in\mathbb{X}} \quad & \sum_{k = 0}^{N} J(\b{x}_k, \b{u}_k) \label{eq:guided-cost}\\
        \textrm{s.t.} \quad \quad & \b{x}_{k + 1} = f(\b{x}_k, \b{u}_k), \: \forall k,\\
        &\b{x}_{0} = \b{x}_{\textrm{init}}\\
        &g(\b{x}_k, \b{o}^j_k) \leq 0 \ \forall k, j\\
        &g_H(\b{x}_k, \b{o}^j_k, \b{\tau}_{i, k}) \leq 0 \ \forall k, j.\label{eq:homotopy_constraints_opt}
    \end{align}
\end{subequations}
The topology constraints and initialization of the optimization realize the local planning mapping of Eq.~\eqref{eq:local-planner} which, as a function of the guidance trajectory, returns a distinct local optimal trajectory (see Fig.~\ref{fig:method-schematic}). 

\subsection{Enforcing Consistency over Time}\label{sec:constraint-validation}
Our proposed method computes distinct trajectories through two algorithmic features: The guidance trajectories are distinct with respect to homotopy comparison function~\eqref{eq:homotopy_comparison} and the homotopy constraints~\eqref{eq:homotopy_constraints_opt} ensure that the local planner does not change the homotopy class during optimization. We illustrate the necessity of these two components with an example.

Consider a planning scenario with a robot and static obstacle (both at $y = 0$) and the reference path at $y = 1$. We plan the robot's trajectory for a low and high weight on following the path\footnote{The path following weight is the contouring weight from~\cite{brito_model_2019}.}.
Fig.~\ref{fig:constraint-ablation} shows the planned trajectories after optimization. Without homotopy comparison~\eqref{eq:homotopy_comparison} (see Fig.~\ref{fig:no-homotopy}), guidance trajectories are not distinct and lead to identical optimized trajectories. Without homotopy constraints~\eqref{eq:homotopy_constraints_opt} (see Fig.~\ref{fig:no-constraints}), trajectories are distinct for a low path following weight but become identical by increasing the path following weight. This is possible as the final state of the optimized trajectory is free to move to the other side of the obstacle. Our proposed method (see Fig.~\ref{fig:no-ablation}) maintains the two trajectories in both cases, irrespective of the tuning of the objective function.

\subsection{Decision Making}\label{sec:decision_making}
The robot can only execute one trajectory. Since the cost function of the local planners~\eqref{eq:guided-cost} matches that of the original trajectory optimization~\eqref{eq:local_objective}, the quality of the guided plans are directly comparable%
\footnote{As trajectory end points can be distinct, their quality needs to be represented in the cost function. We include a terminal cost that accounts for the deviation of the end point from the reference path.} 
through their optimal costs $J_i^*$. %

The local planners output $P$ optimized trajectories
\begin{equation}
    \mathcal{T}_P^* = \left\{\b{\tau}^*_1, \hdots, \b{\tau}^*_{P}\right\}.
\end{equation}
Since each local planner minimizes the same cost function, the lowest cost trajectory\footnote{We set $J_i^*=\infty$ when the optimization is infeasible.}
\begin{equation}
    \b{\tau}_i^*, \quad i = \textrm{argmin}_i \ J_i^*, \label{eq:direct-decision}
\end{equation}
is the best trajectory under the specified objective. We refer to executing $\b{\tau}_i^*$ as obtained from \eqref{eq:direct-decision} as the \textit{minimal cost} decision.

In practice, frequently switching the homotopy class of the executed trajectory can degrade motion planning performance and lead to collisions even if, in each time instance, the selected trajectory attains the lowest cost. %
We therefore consider a generalization of the decision-making process where the previously selected trajectory is given precedence. This is possible as we maintain a consistent set of trajectories in distinct homotopy classes where the previously executed trajectory is marked. 
This \textit{consistent} decision is given by
\begin{equation}
    \b{\tau}_i^*, \quad i = \textrm{argmin}_i \ w_i J_i^*, \label{eq:consistent-decision} %
\end{equation}
where $w_i = c_i$ if this trajectory was previously selected, with $c_i$ a constant $0 \leq c_i \leq 1$, and $w_i = 1$ otherwise. If $c_i = 0$, then the planner will pick the trajectory with the same homotopy class of the previous iteration, while for $c_i = 1$, we recover the minimal cost decision. In practice, this decision-making scheme improves navigation behavior over consecutive iterations.

\subsection{Theoretical Analysis}\label{sec:analysis}
In the following, we formalize to what extent our proposed planner resolves the nonconvexity of the free space. First, note that due to the cost function and/or nonlinear robot dynamics, the trajectory optimization in Eq.~\eqref{eq:local_optimization} remains nonconvex, even when it is constrained to stay in a single homotopy class. There may therefore be multiple local optima in each homotopy class. This means that the proposed planner does not provably return a globally optimal solution to the optimization in Eq.~\eqref{eq:local_optimization}. We propose instead a weaker notion of globalization\footnote{In the following, with some abuse of notation, $J(\b{\tau})$ refers to the optimal cost of the optimization initialized with trajectory $\b{\tau}$.}.

\begin{definition}
    (Homotopy Globally Optimal) Denote the highest-cost local-optimum of optimization~\eqref{eq:local_optimization} in homotopy class $i$ as $\b{\tau}^-_i$. A trajectory $\b{\tau}$ is said to be a Homotopy Globally Optimal (HGO) solution if its cost is lower or equal to that solution in each homotopy class, that is, if $J(\b{\tau}) \leq J(\b{\tau}^-_i), \ \forall i$, for all homotopy classes that admit a feasible trajectory.
\end{definition}
To prove when the proposed scheme computes an HGO solution, we pose three conditions. These conditions link the solution of Eq.~\eqref{eq:guided_local_optimization} to that of Eq.~\eqref{eq:local_optimization}.

\textbf{Condition 1.} The homotopy constraints are not active ($g_H(\b{x}_k, \b{o}_k^j, \b{\tau}_{i, k}) < 0 \, \ \forall i, k, j$) in the final solution of \eqref{eq:guided_local_optimization}. %

\textbf{Condition 2.} The guidance planner finds a trajectory in each homotopy class where a dynamically feasible trajectory exists.

\textbf{Condition 3.} The executed trajectory is selected via \eqref{eq:direct-decision}.

\begin{theorem}
If Conditions 1-3 hold, T-MPC is HGO for optimization problem \eqref{eq:local_optimization}.
\end{theorem}
\begin{proof}
    Under Condition $1$, the solution for each optimization~\eqref{eq:guided_local_optimization}, $\b{\tau}_i^*$, is locally optimal for \eqref{eq:local_optimization} since homotopy constraints \eqref{eq:homotopy_constraints_opt} are the only distinction between the two problems. Therefore, $J(\b{\tau}_i^*) \leq J(\b{\tau}_i^-)$. If Condition $2$ is satisfied, then $\mathcal{T}_P$ contains a guidance trajectory in every feasible homotopy class. Therefore, under Condition $3$, the final trajectory $\b{\tau}^*$ executed by T-MPC satisfies $J(\b{\tau}^*) \leq J(\b{\tau}_i^*) \leq J(\b{\tau}_i^-), \ \forall i$ and we obtain the HGO property.
\end{proof}

This shows under what conditions the proposed planner finds a provably HGO trajectory. Although these conditions are useful for analysis, they are not necessarily satisfied in practice. Condition $1$ can fail if there is no local optimum in a homotopy class or when the linearization around the guidance trajectory restricts the optimization.

Condition $2$ is hard to guarantee in crowded environments. In $2$-D navigation with $M$ obstacles, there can be $2^M$ homotopy classes that do not wind around obstacles. Although robot dynamic constraints and bundled obstacles may reduce this amount in practice, the number of classes can still be too large. Limiting the planner to $P$ classes allows us to plan in real-time, but the executed trajectory may not be HGO. 

Condition $3$ ensures that the lowest cost trajectory is executed but can lead to non-smooth driving behavior over consecutive iterations and it may be preferable to use~\eqref{eq:consistent-decision} instead.

While these conditions may not always be satisfied and the HGO property is not provably obtained in each iteration, we will show that the proposed planner always improves on the local planner in isolation.

\subsection{Non-Guided Local Planner in Parallel}\label{sec:non-guided_planner}
The constraints and initialization provided by the guidance planner allow the local planner to escape poor local optima. 
Once the planner is in the correct homotopy class, the restrictions imposed by the guidance planner (i.e., homotopy constraints) may degrade performance. For this reason, we consider an extension of the proposed planner where the regular local planner without guidance (i.e., the optimization in Eq.~\eqref{eq:local_optimization}) is added to the set of parallel guided local planners. Since this planner is less restricted and does not rely on the global planner, it can occasionally find a better solution. 

Next to practical benefits, this allows us to trivially establish that the proposed scheme does not achieve a higher cost solution than the local planner in isolation.

\begin{theorem}\label{theorem:optimality}
    Consider the planner in Fig.~\ref{fig:method-schematic} that includes a non-guided planner with solution $\bar{\tau}^*$. If a trajectory is selected according to \eqref{eq:direct-decision}, then $J(\b{\tau}^*) \leq J(\bar{\b{\tau}}^*)$.
\end{theorem}
\begin{proof}
    Decision \eqref{eq:direct-decision} picks the lowest cost solution from $J(\b{\tau}_0^*), \hdots, J(\b{\tau}_P^*), J(\bar{\b{\tau}}^*)$, which cannot exceed $J(\bar{\b{\tau}}^*)$.
\end{proof}
If the guided plans are always higher or equal cost compared to the non-guided planner, then this planner architecture reduces to the local planner (the non-guided planner is always selected). If they ever have a lower cost, then guidance \textit{must} improve the planner in the sense that it reduces the cost of the executed trajectory. We will show in the following section that the latter holds true. We refer to the method where the non-guided local planner is added in parallel as T-MPC\texttt{++}.

\subsection{Computation Time Analysis}\label{sec:runtime_analysis}
T-MPC plans guidance trajectories before optimization. In the following, we analyze the computational complexity of the guidance planner (see Algorithm~\ref{alg:guidance_planner}), considering the number of PRM samples $n$, obstacles $M$ and distinct trajectories $P$.
\paragraph*{Visibility-PRM} The time complexity of regular Visibility-PRM is dominated by the visibility check. When adding a node, it checks its visibility in the worst case to all nodes (if all nodes are Guards), where each visibility check considers all obstacles. Its time complexity therefore is $O(n^2M)$. We additionally verify that new connections are distinct. The time complexity of a single homotopy comparison is $O(M)$: the H-signature or winding numbers are evaluated for each obstacle. The homotopy class is compared roughly $n$ times if a new distinct connection neighbors all connectors. Its time complexity therefore is $O(n^2M)$ and does not change the time complexity of Visibility-PRM.
\paragraph*{DepthFirstSearch} Each node links to at most one goal. Hence searching the graph for at most $P$ paths to each goal at worst considers each node once. Its time complexity is $O(n)$.
\paragraph*{FilterAndSelect} Sorting $P$ trajectories has time complexity $O(P\log{P})$. Filtering homotopy distinct trajectories from the sorted list must compare a trajectory to $P$ others in the worst case and has time complexity $O(P^2M)$.
\paragraph*{IdentifyAndPropagate}Similarly, comparing the homotopy class of new and existing trajectories has time complexity $O(P^2M)$. Propagating the graph has time complexity $O(n)$.
\paragraph*{Total}The time complexity of the guidance planner is $O((n^2 + P^2)M)$. In practice, this time complexity can be approximated by $O(n^2M)$, given that the number of relevant homotopy classes is typically small (i.e., $n \gg P$). Thanks to the propagation of nodes from the previous iteration, we find that $n$ can be relatively small as well (e.g., $n < 100$). For our use case, a relatively small $n$ and $P$ are usually sufficient to construct a sparse graph from which the relevant homotopy classes can be extracted.

\section{SIMULATION RESULTS}\label{sec:simulation}
In the following, we compare our planner against several baselines on a mobile robot navigating among pedestrians.

\subsection{Implementation}\label{sec:implementation}
Our implementation for T-MPC is written in C\texttt{++}/ROS and will be released open source\footnote{See \url{https://github.com/tud-amr/mpc_planner}}. We will also release the guidance planner as a standalone package.

For the deterministic simulations we implement the optimization-based planner LMPCC~\cite{brito_model_2019} as local planner. The robot dynamics follow second-order unicycle dynamics~\cite{siegwart_introduction_2011}. Its objective, with weights $w$, is given by\footnote{We do not use the repulsive forces around obstacles from~\cite{brito_model_2019} as they lead to more conservative plans and slow down the optimization problem.}
\begin{equation}
    J = w_cJ_c + w_lJ_l + w_vJ_v + w_{\omega}J_{\omega} + w_aJ_a
\end{equation}
for each time instance $k$ in the horizon $N$. Herein, $J_c, J_l$ are the contour and lag error used to follow the reference path, $J_v = ||v - v_{\textrm{ref}}||_2^2$ tracks a desired velocity and $J_{\omega} = ||\omega||_2^2, \ J_a = ||a||_2^2$ weigh the control inputs consisting of the rotational velocity $\omega$ and acceleration $a$. Collision avoidance constraints are imposed with $g(\b{x}_k, \b{o}_k^j) \leq 0$,
\begin{equation}
    g(\b{x}_k, \b{o}_k^j) = 1-(\Delta\b{p}_k^j)^T\b{R}(\phi)^T\begin{bmatrix}\frac{1}{r^2} & 0 \\ 0 & \frac{1}{r^2}\end{bmatrix}\b{R}(\phi)(\Delta\b{p}_k^j), \label{eq:deterministic-collision-constraints}
\end{equation}
here $\Delta\b{p}_k^j = \b{p}_k - \b{o}_k^j$, $\b{R}(\phi)$ is a rotation matrix with orientation $\phi$ of the robot and $r = r_{\textrm{robot}} + r_{\textrm{obs}}$. These nonconvex constraints directly formulate that the robot region should not overlap with that of the obstacles. We solve each parallel local optimization with Forces Pro~\cite{domahidi_forces_2014}. Parameters of the full planner are listed in Table~\ref{tab:settings}. Weights of the guidance and local planners are manually tuned. The planning scheme, including guidance and local planners, is updated in each iteration in a receding horizon manner.

\begin{table}[b]
\centering
\caption{Experimental settings.}
\label{tab:settings}
\begin{tabular}{|>{\raggedright}p{0.13\columnwidth}|>{\centering\arraybackslash}p{0.15\columnwidth}|>{\raggedright\arraybackslash}p{0.52\columnwidth}|}
\hline
\textbf{Parameter Name} & \textbf{Parameter Value} & \textbf{Parameter Description} \\ \hline
$N$                  & $30$                 & Global and local planner horizon        \\ \hline
$\Delta T$                    & $0.2$ s                 & Integration time step   \\ \hline
$h$ & $0.05$ s                & Planning time step   \\ \Xhline{1.8pt}
$n$                   & $30$                 & Visibility-PRM sample limit      \\ \hline
$T_{\textrm{max}}$                   & $10$ ms                & Visibility-PRM time limit      \\ \hline
Eq.~\ref{eq:homotopy_comparison}                   & H-signature                 &  Homotopy comparison function     \\ \hline
$P$                 & $4$                 &  \# of distinct guidance trajectories     \\ \hline
$G$                   & $5\times 5$                 &  Grid of goals (longitudinal $\times$ lateral)     \\ \Xhline{1.8pt}
$r$                   & $0.725$ m              &  Combined obstacle and robot radius      \\ \hline
$w_c$                   & $0.05$             &  Optimization contouring weight      \\ \hline
$w_l$                   & $0.75$             &  Optimization lag weight      \\ \hline
$w_v$                   & $0.55$             &  Optimization velocity tracking weight      \\ \hline
$w_{\omega}$                   & $0.85$             &  Optimization rotational velocity weight      \\ \hline
$w_a$                   & $0.34$             &  Optimization acceleration weight      \\ \Xhline{1.8pt}
Decision                   & Eq.~\eqref{eq:consistent-decision}             &  Type of decision-making      \\ \hline
$c_i$                   & $0.75$               &  Discount factor for trajectory in previously followed homotopy class      \\ \hline
\end{tabular}
\end{table}

\begin{table}[t]
\centering
\caption{Quantative results for interactive navigation simulations of Sec.~\ref{sec:interactive_results} over $200$ experiments with pedestrian motion \textit{prediction} following a constant velocity model. Task duration (Dur.) and runtime are reported as ``mean (std. dev.)''. Without obstacles, the task duration is $12.9$ s. Best planner performances per column are denoted in \textbf{bold}. \sigsss{Underlined} results indicate that T-MPC\texttt{++} significantly outperforms the respective method as tested with a U-test for a significance value of $p=0.001$.}

\resizebox{0.49\textwidth}{!}{\begin{tabular}{|l|l|c|c|l|}
\hline\textbf{\# Ped.} & \textbf{Method} & \textbf{Dur. [s]} & \textbf{Safe (\%)} & \textbf{Runtime [ms]} \\\hline
$0$ & - & 12.9 (0.0) & - & - \\\hline
\multirow{6}{*}{4} & Fren\'et-Planner~[6] & \sigsss{14.0} (0.9) & 77 & 11.6 (2.6) \\
&TEB Local Planner~[8] & \textbf{13.0} (1.1) & \textbf{100} & \textbf{5.8} (5.2) \\
&LMPCC~[3] & 13.1 (0.4) & 98 & 11.3 (2.6) \\
&Guidance-MPCC~[32] & \textbf{13.0} (0.4) & 92 & 13.9 (1.3) \\
&T-MPC (ours) & \textbf{13.0} (0.2) & \textbf{100} & 18.3 (3.5) \\
&T-MPC\texttt{++} (ours) & \textbf{13.0} (0.1) & \textbf{100} & 19.4 (3.7) \\\hline
\multirow{6}{*}{8} & Fren\'et-Planner~[6] & \sigsss{15.1} (1.7) & 64 & 11.6 (2.4) \\
&TEB Local Planner~[8] & 13.8 (1.7) & \textbf{98} & \textbf{7.5} (5.1) \\
&LMPCC~[3] & \sigsss{13.8} (1.3) & 96 & 13.7 (4.2) \\
&Guidance-MPCC~[32] & \textbf{13.2} (0.7) & 92 & 13.4 (1.4) \\
&T-MPC (ours) & \sigs{13.3} (0.7) & 96 & 20.2 (4.8) \\
&T-MPC\texttt{++} (ours) & \textbf{13.2} (0.6) & 96 & 21.4 (4.9) \\\hline
\multirow{6}{*}{12} & Fren\'et-Planner~[6] & \sigsss{16.5} (2.4) & 42 & 14.1 (6.3) \\
&TEB Local Planner~[8] & \sigsss{14.9} (2.4) & 92 & \textbf{8.7} (5.4) \\
&LMPCC~[3] & \sigsss{14.0} (1.5) & 90 & 12.9 (4.5) \\
&Guidance-MPCC~[32] & \textbf{13.6} (1.1) & 86 & 13.6 (1.6) \\
&T-MPC (ours) & \sigsss{14.1} (1.3) & 90 & 18.3 (5.1) \\
&T-MPC\texttt{++} (ours) & \textbf{13.6} (1.0) & \textbf{93} & 20.1 (5.4) \\\hline
\end{tabular}}
\label{tab:pedsim_cvpredictions}
\end{table}

\subsection{Simulation Environment} %
The first simulation environment (see Fig.~\ref{fig:eye-catcher}) consists of a mobile robot (Clearpath Jackal) moving through a $6$ m wide corridor with up to $12$ pedestrians. The robot follows the centerline with a reference velocity of $2$ m/s and is controlled at $20$ Hz. The pedestrians follow the social forces model~\cite{helbing_social_1995} using implementation~\cite{gloor_pedsim_2016}. They interact with other pedestrians and the robot and are aware of the walls. We use a constant velocity model to predict the future pedestrian positions for the planner. The pedestrians have a radius of $0.3$ m. We specify a radius of $0.4$ m in the planners to account for discretization effects, allowing us to clearly identify collisions. Pedestrians spawn on two sides of the corridor with the objective to traverse the corridor. The random start and goal locations are the same for each planner.

\subsection{Comparison to Baselines}\label{sec:interactive_results}
We compare T-MPC and T-MPC\texttt{++} against four baselines. Baselines are selected on the availability of an open-source implementation and their application to navigation in $2$D dynamic environments. We consider the following baselines:
\begin{itemize}
    \item \textbf{Motion Primitives} \textit{(global planner)}~\cite{werling_optimal_2010}: A non-optimization-based global planner that respects the robot dynamics.
    \item \textbf{TEB Local Planner} \textit{(topology-guided planner)}~\cite{rosmann_integrated_2017}: One of the most used local planners in the ROS navigation stack~\cite{rosmann_ros_2023} that considers multiple homotopy classes.
    \item \textbf{LMPCC} \textit{(local planner)}~\cite{brito_model_2019}: An open-source non parallelized MPC (see Sec.~\ref{sec:implementation}). We supply the previous solution shifted forward in time as the initial guess of the optimization.
    \item \textbf{Guidance-MPCC} \textit{(topology-guided planner)}~\cite{de_groot_globally_2023}: Our previous conference work. We updated the guidance planner to that used in this work to make it more competitive with T-MPC\texttt{++}.
\end{itemize}
We use the same weights for the MPC planners (LMPCC, Guidance-MPCC, T-MPC, T-MPC\texttt{++}). Baseline actuation limits and tracking objectives are adapted to match the MPC objectives. TEB Local Planner tuning uses its default but with increased collision avoidance weight (from $10$ to $20$) to decrease collisions in crowded environments. 

We perform the simulations with $4, 8$ and $12$ pedestrians.

\textbf{Evaluation Metrics.} We compare the planners on the following metrics.
\begin{enumerate}
    \item \textit{Task Duration:} The time it takes to reach the end of the corridor.
    \item \textit{Safety:} The percentage of experiments in which the robot does not collide with pedestrians or the corridor bounds.
    \item \textit{Runtime:} Computation time of the control loop.
\end{enumerate}
We note that collisions in simulation may not correspond to collisions in practice but do provide insight into the safety of the planners (more details in Sec.~\ref{sec:discussion}).
\begin{figure}[t]
    \centering
    \includegraphics[width=0.48\textwidth]{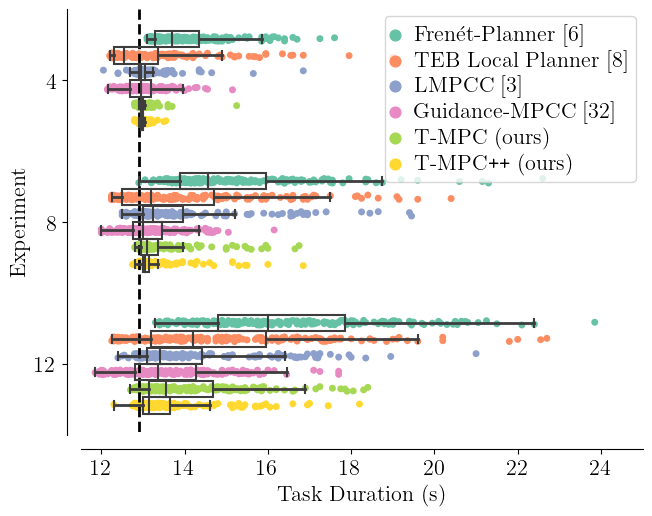}
    \caption{Visualization of the task duration (i.e., the time taken to reach the goal) in Table~\ref{tab:pedsim_cvpredictions}. The dashed vertical line denotes the task duration without obstacles. Our method achieves the smallest variation in task duration and the shortest task duration in crowded environments.}
    \label{fig:pedsim_notinteractive_duration}%
\end{figure}
\begin{figure}[t]
    \centering
    \begin{subfigure}{0.4\textwidth}
        \centering
        \includegraphics[width=\textwidth]{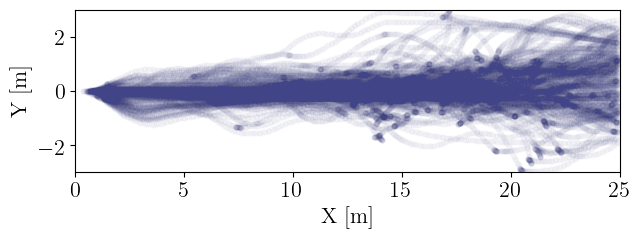}
        \caption{TEB-Planner~\cite{rosmann_integrated_2017}.}
        \label{fig:trajectories_teb}
    \end{subfigure}

    \begin{subfigure}{0.4\textwidth}
        \centering
        \includegraphics[width=\textwidth]{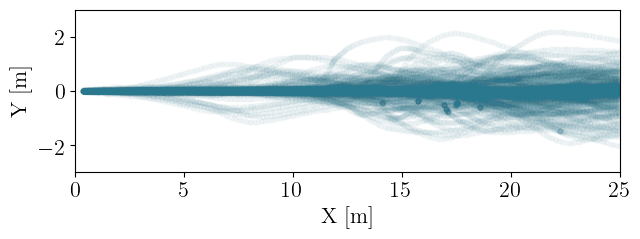}
        \caption{T-MPC\texttt{++} (ours)}
        \label{fig:trajectories_tmpc}
    \end{subfigure}
    \caption{Trajectories of $200$ experiments with $12$ pedestrians for the TEB-Planner and T-MPC\texttt{++}. Our method results in smoother and more consistent robot navigation.}
    \label{fig:noninteractive_trajectories}
\end{figure}

The simulations are performed on a laptop with an Intel i9 CPU@2.4GHz $16$ core CPU. Our implementation of T-MPC and T-MPC\texttt{++} use $P$ and $P+1$ CPU threads, respectively. We terminate threads when they exceed the control period of $50$ ms and in this case select the best trajectory out of the completed optimization problems.

The results over $200$ experiments are summarized in Table~\ref{tab:pedsim_cvpredictions} and the task duration is visualized in Fig.~\ref{fig:pedsim_notinteractive_duration}. The Motion Primitives planner is not safe and has a significantly longer task duration than the other planners even in the least crowded case. In the two more crowded environments, LMPCC has a significantly longer task duration than T-MPC\texttt{++} and is less safe, in part because LMPCC plans a single trajectory. 
When the optimization becomes infeasible, it often does not recover fast enough to avoid oncoming obstacles. TEB Local Planner is marginally safer than LMPCC and maintains competitive average task durations to the other methods in the four and eight pedestrian scenarios under less computational cost. In the crowded scenario, T-MPC\texttt{++} completes the task significantly faster. Additionally, in all cases, T-MPC\texttt{++} has a much smaller standard deviation of the task duration indicating that its behavior is more consistent (visible also in Fig.~\ref{fig:pedsim_notinteractive_duration}). The TEB Local Planner soft constrains collision avoidance which, in crowded environments, leads the robot into poor behaviors (e.g., reversing) due to the shape of the cost function. To further compare T-MPC\texttt{++} and the TEB Local Planner, we visualize their trajectories in Fig.~\ref{fig:noninteractive_trajectories}. Our proposed planner results in smoother and more consistent trajectories and can follow the reference path more closely. We quantitatively compare the smoothness of the planners through the standard deviation on second-order input commands. The deviation on acceleration ($\sigma_a$) and rotational acceleration ($\sigma_{\alpha}$) for TEB Local Planner are higher ($\sigma_a = 0.16$, $\sigma_{\alpha} = 0.12$) than for T-MPC\texttt{++} ($\sigma_a = 0.04$, $\sigma_{\alpha}=  0.05$).

Out of the guidance planners, Guidance-MPCC attains the same mean task duration as T-MPC\texttt{++} but is less consistent (high std. dev.) because the guidance planner, which does not account for the robot dynamics, determines the planner's behavior. This results in larger tracking errors under the same cost function, for example in the $12$ pedestrian case, the mean path and velocity errors of Guidance-MPCC are $0.45$m and $0.47$m/s, respectively, compared to values of $0.20$m and $0.42$m/s for T-MPC\texttt{++}. 
It also leads to more collisions than LMPCC. \mbox{T-MPC} is generally faster than LMPCC, except for the $12$ pedestrian case. In this environment, the local planner may find solutions that the guidance planner did not, given that the space is cluttered. T-MPC++ demonstrates superior navigation performance over the other planners: it is significantly faster (with the exception of Guidance-MPCC), varies less in its task duration (lower std. dev.) and is safer in almost all cases (the TEB Local Planner is safer in the $8$ pedestrian case). T-MPC\texttt{++} has higher computational demands than the other planners. Compared to the other MPC planners, T-MPC\texttt{++} first computes guidance plans. We measured that this step takes approximately $5$ ms on average (included in the runtime of Table~\ref{tab:pedsim_cvpredictions}) in all scenarios.

\subsection{Crowded Baseline Comparison}\label{sec:results_crowded}
We further compare T-MPC\texttt{++} against TEB Local Planner in a square-shaped crowded environment with $50$ pedestrians. The robot's task is to move diagonally through the environment at $v_{\textrm{ref}}=1.5$ m/s. We consider the $12$ pedestrians closest to the plan, with preference for nearby pedestrians in both methods. Pedestrians are removed when they reach the goal to prevent unpredictable turns. Table~\ref{tab:openspace} presents the results. T-MPC\texttt{++} is significantly faster and the standard deviation of the task duration is less than half that of TEB Local Planner. TEB Local Planner is on average computationally faster as it does not compute a new plan in each iteration. When it does compute a plan, its computation time can exceed the planning frequency of $20$ Hz as shown by the maximum in Table~\ref{tab:openspace} (it exceeded $50$ ms in $0.26$\% of its iterations). In contrast, T-MPC\texttt{++} is explicitly limited to $50$ ms such that its maximum computation time remains below $50$ ms.

\begin{table}[t]
\centering
\caption{Quantative results in crowded environment of Sec.~\ref{sec:results_crowded} over $200$ experiments. Notation follows that of Table~\ref{tab:pedsim_cvpredictions}. Without obstacles, the task duration is $20.6$ s. The runtime is denoted as ``mean (max)''.}
\resizebox{0.49\textwidth}{!}{\begin{tabular}{|l|l|c|c|l|}
\hline\textbf{\# Ped.} & \textbf{Method} & \textbf{Dur. [s]} & \textbf{Safe (\%)} & \textbf{Runtime$^*$ [ms]} \\\hline
$0$ & - & 20.6 (0.0) & - & - \\\hline
\multirow{2}{*}{50} & TEB Local Planner~[8] & \sigsss{22.4} (2.4) & \textbf{92} & \textbf{9.4} (177.1) \\
&T-MPC\texttt{++} (ours) & \textbf{21.0} (0.9) & \textbf{92} & 21.2 (46.9) \\\hline
\end{tabular}}
\label{tab:openspace}
\end{table}

\subsection{Sensitivity Studies}
To provide more insight into the key parameters of our approach, we study their sensitivity.

\subsubsection{Sensitivity to the Number of Trajectories $P$}
To study how the number of guidance trajectories impacts the task duration, we run $100$ experiments in the $12$ pedestrians environment and compute $P = 0, \hdots, 6$ guidance trajectories. The $P=0$ case corresponds to the non-guided local planner, LMPCC~\cite{brito_model_2019}.

Fig.~\ref{fig:path-comparison} displays statistics on task duration. We observe from Fig.~\ref{fig:path-comparison-with} that guidance trajectories reduce the task duration compared to the local planner. Fig.~\ref{fig:path-comparison-infeasibility} shows for both T-MPC and T-MPC\texttt{++} how often the planner becomes infeasible. It indicates that the availability of at least two plans makes it more likely that a trajectory is found and shows that the non-guided planner added in T-MPC\texttt{++} further improves feasibility.

\subsubsection{Sensitivity to the Consistency Parameter $c_i$}\label{sec:sensitivity-constistency}
We recall that $c_i$ (see Eq.~\eqref{eq:consistent-decision}) expresses our preference to follow the trajectory with the same passing behavior as in the previous planning iteration. We vary $c_i\in[0, 1]$ within its range, including $c_i = 0$ that enforces the robot to follow the trajectory in the previous homotopy class, if it still exists, and $c_i = 1$ that expresses no preference.

Fig.~\ref{fig:sensitivity-consistency} compares task duration and infeasibility. Both extreme values show poor performance. For $c_i = 0$, the task duration increases, while for $c_i = 1$, the planner becomes infeasible more often (it is indecisive). We deployed $c_i=0.75$, which best retains the feasibility of the optimization.
\begin{figure}[t]
    \begin{subfigure}{0.47\textwidth}
        \centering
        \includegraphics[width=\textwidth]{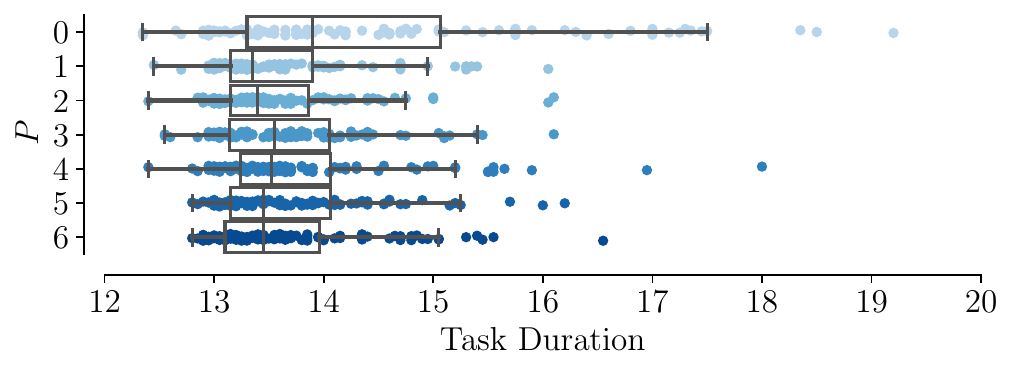}
        \caption{Task duration of T-MPC\texttt{++}.}%
        \label{fig:path-comparison-with}
    \end{subfigure}

    \begin{subfigure}{0.47\textwidth}
        \centering
        \includegraphics[width=\textwidth]{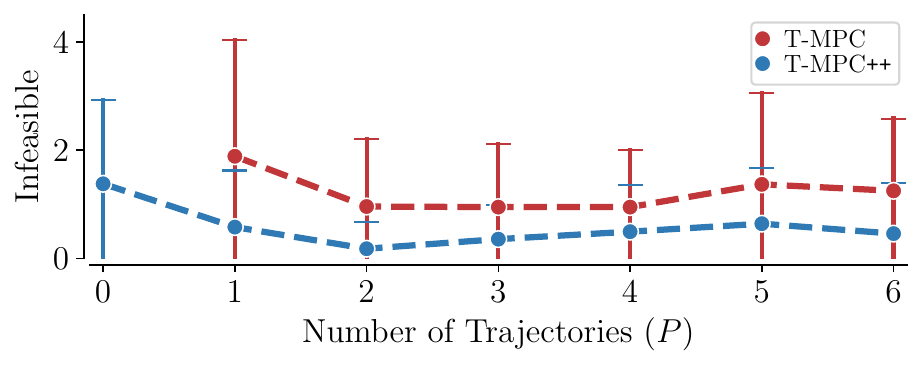}
        \caption{Infeasible iterations (mean and std) per experiment.}%
        \label{fig:path-comparison-infeasibility}
    \end{subfigure}
    \caption{Sensitivity study of the number of guidance trajectories $P$. (a) Task duration of T-MPC\texttt{++}, individual experiments denoted by dots. (b) Mean and standard deviation of the number of control iterations in which the optimization is infeasible per experiment (out of approximately 280 iterations each) for the same simulations.}
    \label{fig:path-comparison}%
\end{figure}

\begin{figure}[t]
    \centering%
    \begin{subfigure}{0.45\textwidth}%
        \includegraphics[width=\textwidth,trim={0pt 0cm 0pt 0pt},clip]{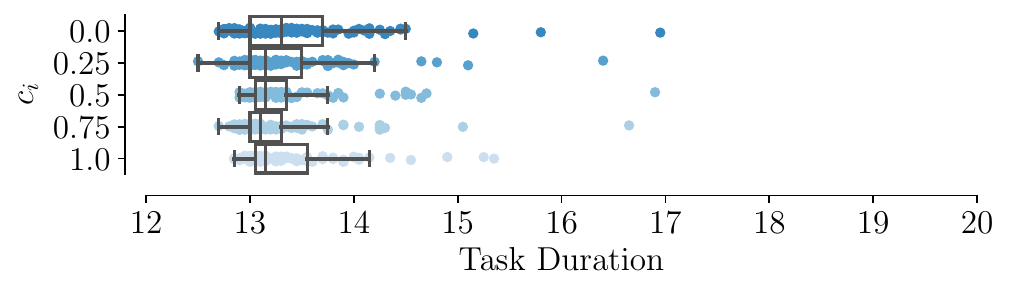}
    \end{subfigure}%

    \begin{subfigure}{0.38\textwidth}%
        \includegraphics[width=\textwidth,trim={0pt 0cm 0pt 0pt},clip]{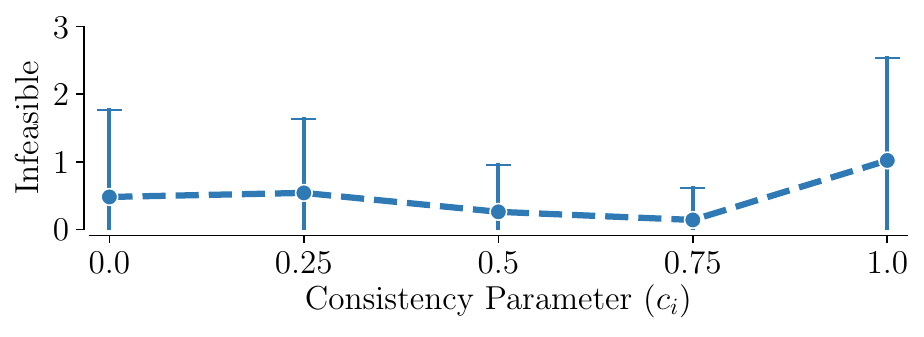}
    \end{subfigure}%

    \caption{Sensitivity study of the consistency parameter $c_i$ comparing the task duration and infeasibility of the planner. Note that a lower $c_i$ makes the planner \textit{more} consistent.}
    \label{fig:sensitivity-consistency}
\end{figure}

\subsection{Empirical Cost Comparison}\label{sec:empirical_cost}
To verify that T-MPC\texttt{++} is able to find lower-cost local optima than the local planner in isolation (as discussed in Sec.~\ref{sec:analysis}), we compare the optimal cost of the executed trajectory attained by LMPCC (local planner), T-MPC and T-MPC\texttt{++} that use identical cost functions. We perform this study in the crowded environment of Sec.~\ref{sec:results_crowded} over $50$ experiments. Table~\ref{tab:cost} indicates that T-MPC finds lower-cost local optima than LMPCC and does so consistently (lower std. dev.). T-MPC\texttt{++} further reduces this cost and its deviation. The average cost of T-MPC\texttt{++} is less than half that of the non-guided planner. %

\begin{table}[t]
\centering
\caption{Comparison of the attained cost in the crowded environment of Sec.~\ref{sec:results_crowded} over $50$ experiments. Notation, including the cost that is tested for significance, follows that of Table~\ref{tab:pedsim_cvpredictions}. The cost excludes infeasible planner iterations.}
\resizebox{0.49\textwidth}{!}{\begin{tabular}{|l|l|c|c|c|l|}
\hline\textbf{\# Ped.} & \textbf{Method} & \textbf{Dur. [s]} & \textbf{Safe (\%)} & \textbf{Cost} & \textbf{Runtime [ms]} \\\hline
\multirow{3}{*}{50} & LMPCC~[3] & \sigss{21.9} (1.7) & 84 & \sigsss{2.25} (9.36) & \textbf{7.1} (4.8) \\
&T-MPC (ours) & \sigs{21.6} (1.5) & 84 & \sigsss{1.60} (6.21) & 20.6 (7.7) \\
&T-MPC\texttt{++} (ours) & \textbf{21.0} (0.9) & \textbf{96} & \textbf{1.05} (4.67) & 21.7 (7.4) \\\hline
\end{tabular}}
\label{tab:cost}
\end{table}

\subsection{T-MPC Under Obstacle Uncertainty}\label{sec:results_uncertainty}
To illustrate that T-MPC applies to different local planner formulations, we deploy T-MPC on top of CC-MPC~\cite{zhu_chance-constrained_2019}. CC-MPC is a local planner that considers the probability of collision with obstacles when their motion is represented by a Gaussian distribution at each time step. We assume that the motion of the obstacles follows the uncertain dynamics
\begin{equation}
    \b{o}^j_{k + 1} = \b{o}^j_k + (\b{v}^j_k + \b{\eta}^j_k) dt, \ \b{\eta}^j_k \sim \mathbb{P}^j_k,
\end{equation}%
Here $\b{v}_k$ is the velocity that follows the social forces model as in previous experiments. The distribution of $\b{\eta}^j_k \in \mathbb{R}^2$ follows a bivariate Gaussian distribution, $\b{\eta}^j_k \sim \mathcal{N}(\b{\mu}^j_k, \b{\Sigma}^j_k)$, where $\b{\mu}^j_k = \b{0}$ and $\b{\Sigma}^j_k = \sigma\b{I}$. In this simulation we set $\sigma = 0.3$. Instead of deterministic collision avoidance constraints~\eqref{eq:deterministic-collision-constraints}, we formulate a chance constraint with risk $0 < \epsilon < 1$
\begin{equation}
    \mathbb{P}\left[||\b{p}_k - \b{o}^j_k||_2^2 \geq r \right] \geq 1 - \epsilon, \ \forall k, j
\end{equation}
that specifies collision avoidance to hold with a probability of $1 - \epsilon$ for each agent and time instance. \mbox{CC-MPC}~\cite{zhu_chance-constrained_2019} reformulates this constraint using the Gaussian $1$-D CDF in the direction of the obstacle. In this formulation, the collision avoidance constraint is linearized and reduces to
\begin{equation*}
    (\b{A}_k^j)^T(\b{p}_k - \b{o}^j_k) - r - r^j \geq \textrm{erf}^{-1}(1 - 2\epsilon) \sqrt{2(\b{A}_k^j)^T\bm{\Sigma}_k^j(\b{A}_k^j)},
\end{equation*}
with $\b{A}_k^j$ as in \eqref{eq:homotopy-constraints} and where $\textrm{erf}^{-1}$ is the inverse standard error function.

We apply T-MPC with the non-guided CC-MPC in parallel (referred to as TCC-MPC\texttt{++}). The uncertainty directly affects the local planner, while the guidance planner only avoids the mean obstacle trajectories. It may happen that some guidance trajectories are not feasible for the local planner.

We deploy both planners in the scenario with $12$ randomized pedestrians and compare the planners under high ($\epsilon=0.1$), medium ($\epsilon=0.01$) and low ($\epsilon=0.001$) risk settings. The results are shown in Table~\ref{tab:uncertainty_results} and visualized in Fig.~\ref{fig:uncertainty_duration}. \mbox{TCC-MPC\texttt{++}} consistently outperforms CC-MPC in isolation, leading to significantly faster and more consistent task completion and fewer collisions. The local planner often collides when it becomes infeasible since it cannot recover. The initialization provided by the guidance planner resolves infeasibility and improves robustness.
\begin{table}[t]
\centering
\caption{Quantative results for simulations with uncertain obstacle motion of Sec.~\ref{sec:results_uncertainty} over $200$ experiments. Notation follows that of Table~\ref{tab:pedsim_cvpredictions}.}
\resizebox{0.49\textwidth}{!}{\begin{tabular}{|l|l|c|c|l|}
\hline\textbf{\#} & \textbf{Method} & \textbf{Task Duration [s]} & \textbf{Safe (\%)} & \textbf{Runtime [ms]} \\\hline
\multirow{2}{*}{High Risk} & CC-MPC~[4] & \sigsss{15.8} (2.3) & 91 & \textbf{17.0} (8.7) \\
&TCC-MPC\texttt{++} (ours) & \textbf{14.1} (0.8) & \textbf{96} & 34.5 (7.9) \\\hline
\multirow{2}{*}{Medium Risk} & CC-MPC~[4] & \sigsss{16.5} (2.7) & 92 & \textbf{17.4} (9.7) \\
&TCC-MPC\texttt{++} (ours) & \textbf{15.1} (1.4) & \textbf{93} & 35.8 (8.2) \\\hline
\multirow{2}{*}{Low Risk} & CC-MPC~[4] & \sigsss{17.2} (2.5) & 90 & \textbf{18.5} (10.2) \\
&TCC-MPC\texttt{++} (ours) & \textbf{16.1} (1.3) & \textbf{97} & 38.1 (7.9) \\\hline
\end{tabular}}
\label{tab:uncertainty_results}
\end{table}

\begin{figure}[t]
    \centering
    \includegraphics[width=0.48\textwidth]{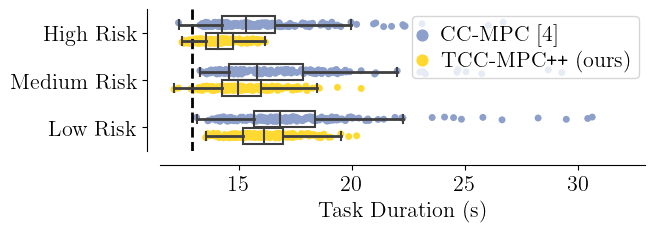}
    \caption{Visualization of the task duration in Table~\ref{tab:uncertainty_results}.}
    \label{fig:uncertainty_duration}%
\end{figure}%

\section{REAL-WORLD EXPERIMENTS}\label{sec:experiments}
We demonstrate the proposed planner in a real-world setting on a mobile robot driving among pedestrians.

\subsection{Experimental Setup}
The experiment takes place in a $5$m$\times 8$m square environment where participants walk among the robot. The robot and pedestrian positions are detected by a motion capture system at $20$Hz. The pedestrian positions are passed through a Kalman filter and constant velocity predictions are passed to the planner. The robot is given a reference path between two opposite corners and turns around once a corner is reached. 

\begin{figure*}
    \centering
    \begin{subfigure}{0.33\textwidth}
        \centering
        \includegraphics[width=\textwidth]{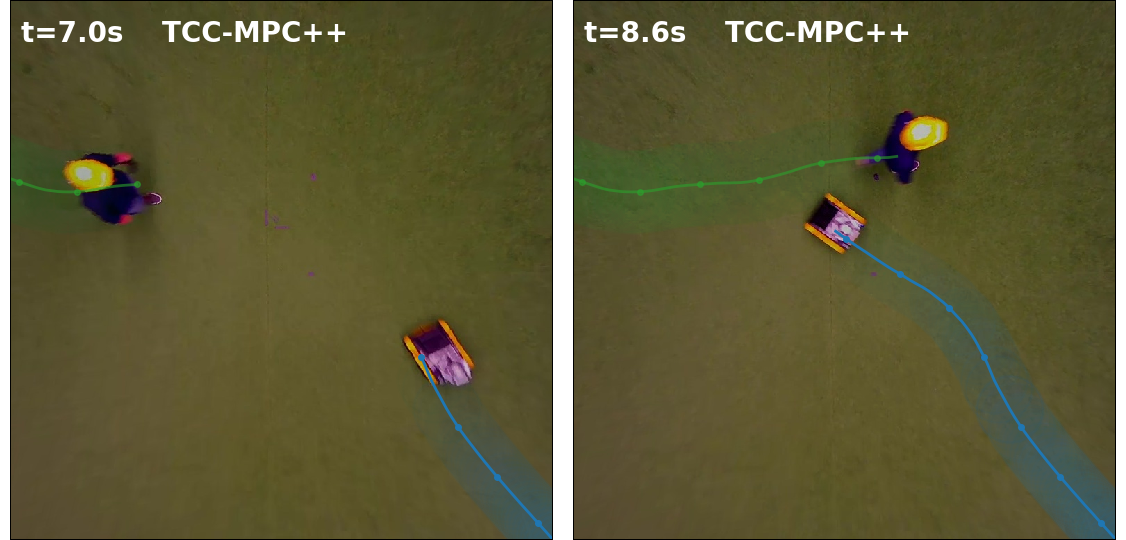}
        \caption{Encounter with a turning pedestrian.}
        \label{fig:1person_a}
    \end{subfigure}
    \hspace{0.03\textwidth}
    \begin{subfigure}{0.33\textwidth}
        \centering
        \includegraphics[width=\textwidth]{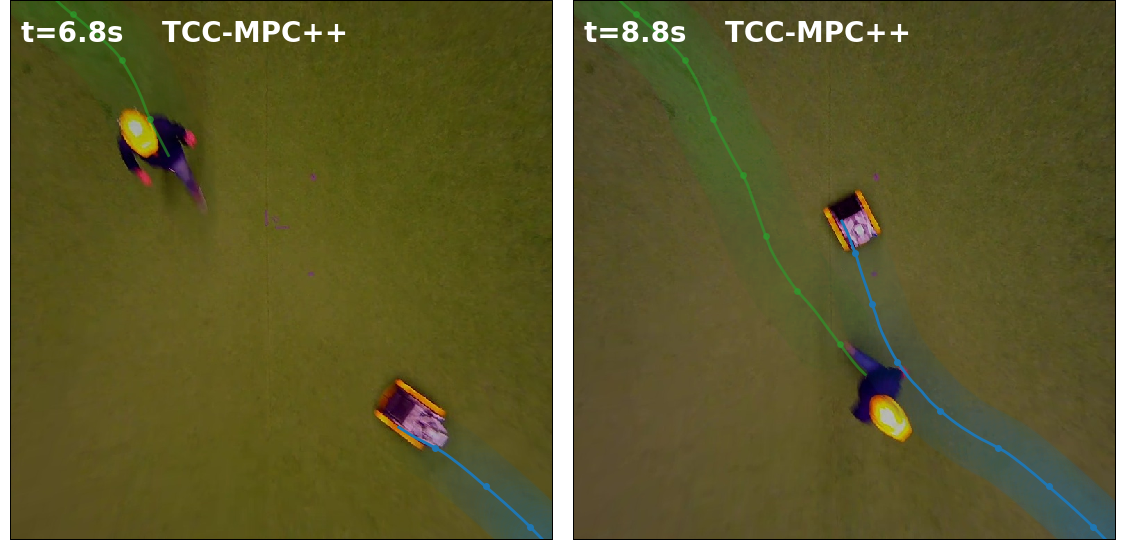}
        \caption{Pedestrian switching intended passing side.}
        \label{fig:1person_b}
    \end{subfigure}
    \caption{Overlayed top view camera images of real-world experiments of TCC-MPC\texttt{++} with one pedestrian. Blue and green overlays denote the robot and pedestrian trajectories, respectively. Timestamps of each image denoted in the upper left corner.}
    \label{fig:1person}
\end{figure*}%
\begin{figure*}%
    \centering
    \begin{subfigure}{0.73\textwidth}%
        \centering%
        \includegraphics[width=\textwidth]{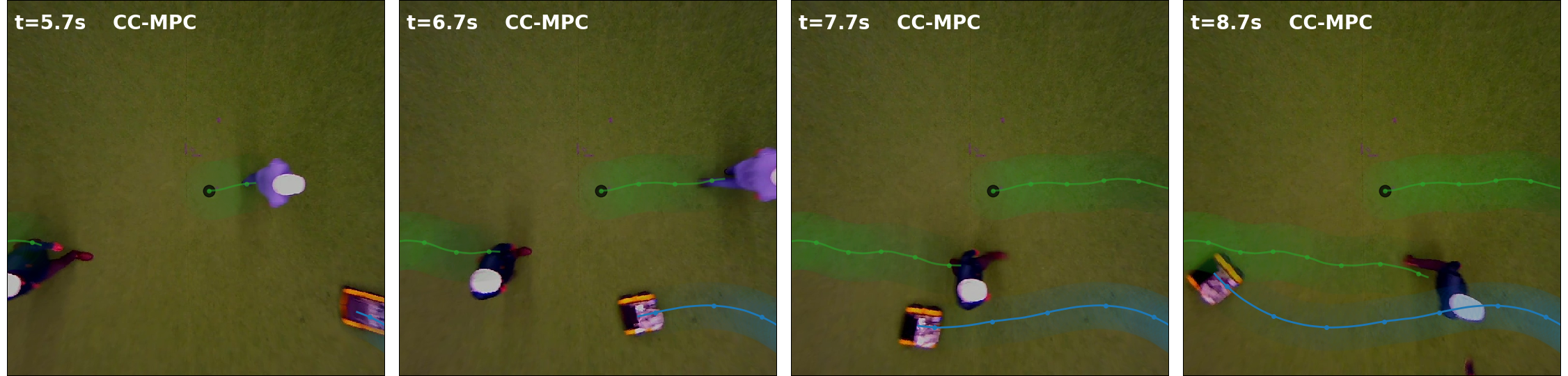}%
        \caption{The baseline fails to find the more efficient trajectory passing between two uncoming pedestrians.}%
        \label{fig:cc-mpc-a}%
    \end{subfigure}
    \begin{subfigure}{0.73\textwidth}%
        \centering%
        \includegraphics[width=\textwidth]{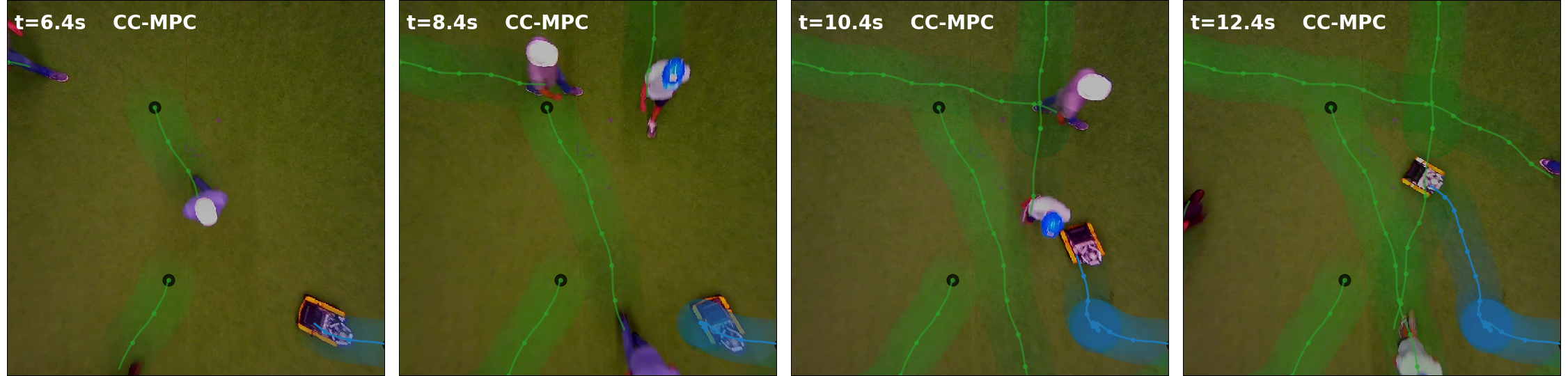}%
        \caption{The baseline becomes infeasible and does not recover quickly (note the time scale).}%
        \label{fig:cc-mpc-b}%
    \end{subfigure}%
    \caption{Trajectories of baseline CC-MPC~\cite{zhu_chance-constrained_2019} overlayed on camera images. Black dots denote pedestrian start positions.}%
    \label{fig:cc-mpc}%
\end{figure*}%

\subsection{One Pedestrian}
In the first set of experiments, a single pedestrian interacts with the robot. We run TCC-MPC\texttt{++} to evade the pedestrian. Fig.~\ref{fig:1person} shows the results of two experiments. In Fig.~\ref{fig:1person_a}, the pedestrian turns and speeds up to pass the robot in front. The robot changes its behavior from passing in front to passing behind to let the pedestrian pass. In Fig.~\ref{fig:1person_b}, the pedestrian changes its intended passing side from the right to the left side of the robot. The planner detects the change in direction and switches sides, passing the pedestrian smoothly.

\begin{figure*}
    \centering
    \begin{subfigure}{0.8\textwidth}
        \centering
        \includegraphics[width=\textwidth]{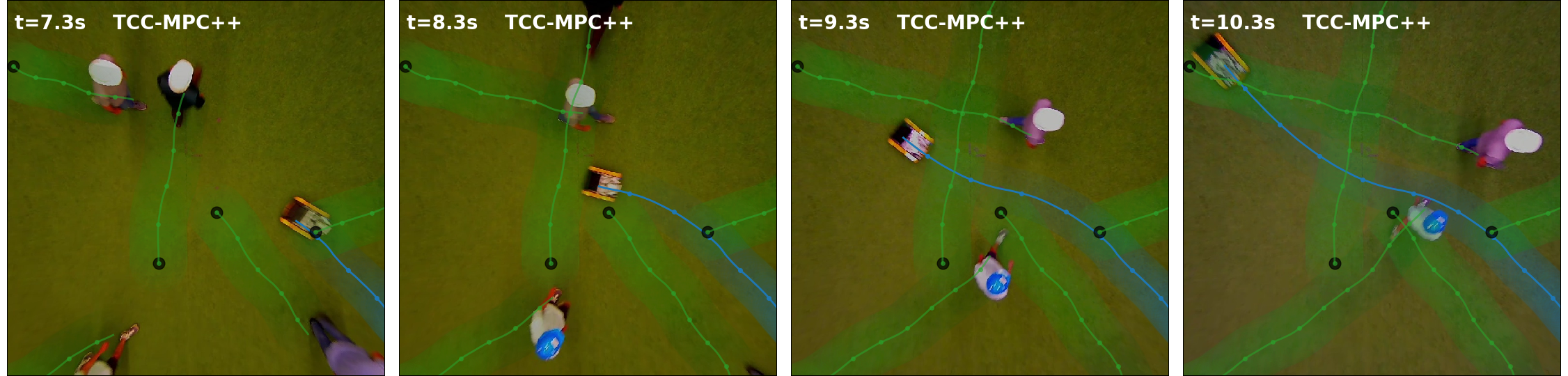}
        \caption{The planner passes between two humans smoothly.}
        \label{fig:ggmpcc-a}
    \end{subfigure}
    \begin{subfigure}{0.8\textwidth}
        \centering
        \includegraphics[width=\textwidth]{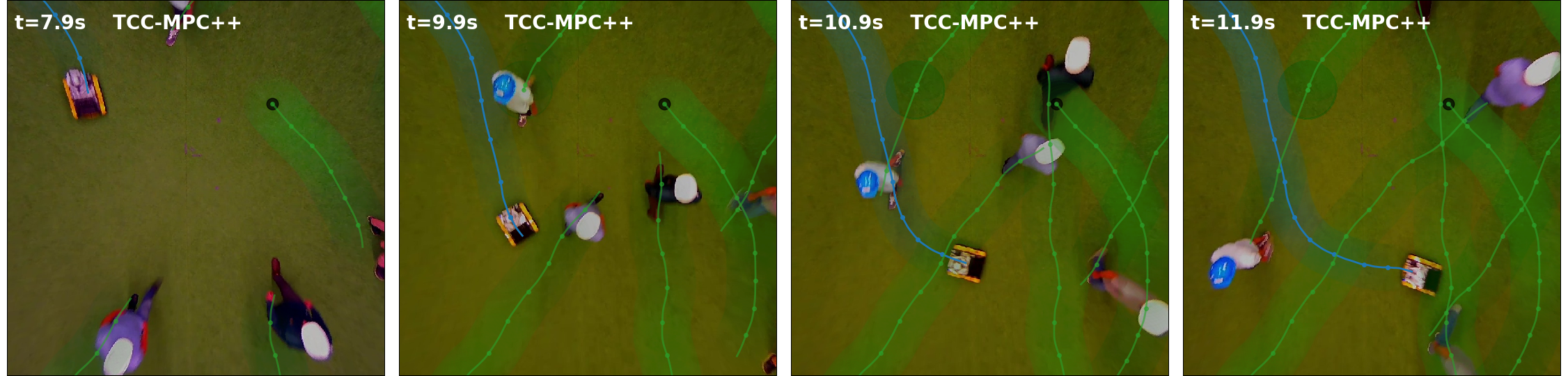}
        \caption{Noticing the oncoming pedestrians, the planner passes behind them smoothly.}
        \label{fig:ggmpcc-b}
    \end{subfigure}
    \begin{subfigure}{0.8\textwidth}
        \centering
        \includegraphics[width=\textwidth]{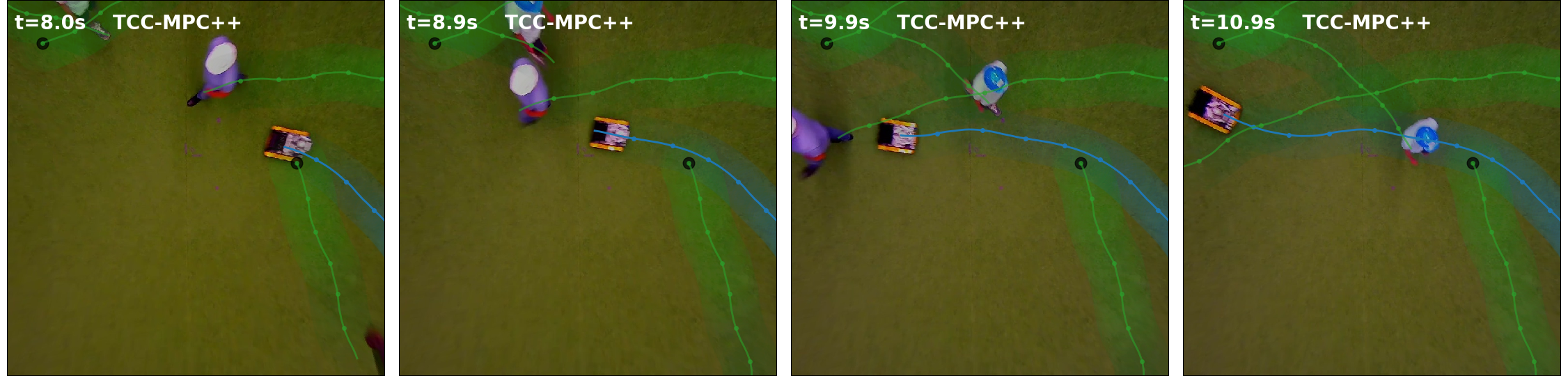}
        \caption{The planner follows another passing pedestrian to pass.}
        \label{fig:ggmpcc-c}
    \end{subfigure}
    \caption{Trajectories of TCC-MPC\texttt{++} overlayed on camera images for three examples.}
    \label{fig:ggmpcc}
\end{figure*}%

\subsection{Five Pedestrians}\label{sec:five_peds}
In the second set of experiments, we run LMPCC, T-MPC++, CC-MPC and TCC-MPC\texttt{++} for $5$ minutes each with $5$ pedestrians in the space. The participants were instructed to walk naturally toward a point on the other side of the lab. Before starting the recorded experiments, participants were asked to walk for $3$ minutes without the robot, to get used to the environment. The participants were not aware of the planner running in each experiment and the order of planners was randomized. Several runs of CC-MPC and TCC-MPC\texttt{++} are visualized in Fig.~\ref{fig:cc-mpc} and Fig.~\ref{fig:ggmpcc}, respectively. The reader is encouraged to watch the associated video~\cite{video_topology-driven_2024}. Fig.~\ref{fig:cc-mpc-a} highlights a case where CC-MPC conservatively avoids two pedestrians, not detecting the pathway in between the pedestrians. In Fig.~\ref{fig:cc-mpc-b}, CC-MPC gets infeasible and is not able to replan fast enough, requiring the robot to wait for a pedestrian to pass. In all experiments of Fig.~\ref{fig:ggmpcc}, the robot passes pedestrians efficiently and smoothly. Fig.~\ref{fig:ggmpcc-a} and~\ref{fig:ggmpcc-b} highlight cases where the planner has to navigate through the crowd and does so successfully. In Fig.~\ref{fig:ggmpcc-c}, we observe the planner falling in line with a pedestrian to pass another pedestrian.

After the experiments, the participants unanimously preferred the planners running in experiments $1$ and $4$, which were both guided planners. In general, participants reported that the guided planners felt safer and more predictable than the non-guided planners.

\section{DISCUSSION}\label{sec:discussion}
Our results have indicated that our proposed global and local planning framework can improve the safety, consistency and time efficiency of the planner. We discuss further insights related to our planning framework in the following.

\subsection{Safety in Dynamic Environments}\label{sec:discussion_collisions}
Planners deployed in the real world must be safe (i.e., collisions are unacceptable) and should not impede humans more than necessary. Table~\ref{tab:pedsim_cvpredictions} indicated that all planners collided at least once in simulation. To identify the source of collisions, we repeated the experiments using the social forces model for both the pedestrian simulation model and the prediction model of the planner (i.e., removing prediction mismatch). The results are summarized in Table~\ref{tab:pedsim_interaction}. Collisions are almost reduced to zero for T-MPC\texttt{++} in this case, which shows that prediction mismatch causes most of the collisions.

\begin{table}[b]
\centering
\caption{Quantative results for the simulations of Sec.~\ref{sec:interactive_results} repeated with the pedestrian motion predictions following the social forces model.}
\resizebox{0.49\textwidth}{!}{\begin{tabular}{|l|l|c|c|l|}
\hline\textbf{\# Peds.} & \textbf{Method} & \textbf{Task Duration [s]} & \textbf{Safe (\%)} & \textbf{Runtime [ms]} \\\hline
\multirow{1}{*}{4} & T-MPC\texttt{++} (ours) & 13.0 (0.1) & 100 & 25.4 (4.9) \\
\multirow{1}{*}{8} & T-MPC\texttt{++} (ours) & 13.2 (0.4) & 100 & 27.7 (6.2) \\
\multirow{1}{*}{12} & T-MPC\texttt{++} (ours) & 13.6 (0.7) & 98 & 26.7 (7.0) \\\hline
\end{tabular}}
\label{tab:pedsim_interaction}
\end{table}

In the real-world experiments, we did not observe collisions. We analyzed the video of the experiments and visually annotated the following instances where pedestrians had to take evasive action: $3$ out of $63$ interactions using LMPCC, $4$ out of $61$ interactions using CC-MPC, $0$ out of $61$ interactions for T-MPC\texttt{++} and $0$ out of $60$ interactions for TCC-MPC\texttt{++}. 
This indicates that pedestrians take direct evasive action when the robot impedes their safety, deviating from the social forces model. Collisions in simulation therefore seem to correspond to cases where the human must take evasive action in practice. %

We also observed in the real world that evasive action was necessary when the baseline planner became infeasible. Safety guarantees provided through constraints only hold when the optimization problem is feasible and can impose danger when no solution is found in time. Our proposed approach reduced this danger by planning more than a single trajectory and we did not observe dangerous cases of infeasibility for T-MPC.

\subsection{Advantages of Parallel Optimization}
Deploying several local planners in parallel makes it more likely that the planner returns a trajectory (in time) as a 
feasible trajectory can be provided by not a single, but several optimization problems. We suspect that this effect is more pronounced when the planning problems are more diverse. In addition, parallelization reduces the maximum computation times. The fastest solved optimization immediately provides a trajectory and other problems can be ignored if necessary. The parallel planner computation time is, at worst, equal to that of a single planner but is almost always faster. Several CPU cores are necessary to parallelize the planner but are usually available. These two advantages inherently improve all performance metrics (e.g., safety and time efficiency) as a solution is more often available. Because redundancy and reduced computation times are key for real-world applications, parallelization may be the key to safely and efficiently deploying optimization-based planners in practice.

\subsection{Selection of the Homotopy Class}\label{sec:discussion_homotopy}
The decision-making in Sec.~\ref{sec:decision_making} used the optimal costs of the local planners to decide which trajectory to execute and preferred the homotopy class of the last followed trajectory. We observed in practice that the robot stayed closer to the reference path and velocity, and passed pedestrians behind rather than in front when necessary. We additionally observed that due to measurement and prediction noise, making a more consistent decision led to better navigation. With high consistency, the robot switches behavior only if the new one is significantly better than the current one. In this way we can ensure a more robust estimation of the cost of the trajectory. While our proposed decision-making led to fast navigation, it ignored social norms. The decision-making could be made more socially compliant by learning to pick the homotopy class that humans take from data (see e.g.,~\cite{martinez-baselga_shine_2024, kretzschmar_socially_2016,rosmann_online_2017}). Finally, as noted by~\cite{rosmann_integrated_2017}, homotopy classes merge and split when obstacles are passed or appear in the planning horizon. Reacting to these events could make the planner more responsive in practice.%

\subsection{Limitations and Future Work}
One of the remaining limitations of the framework is the lack of interaction between the humans and the robot. The social forces model that we used in simulation is interactive but does not accurately model human-robot interactions. It may be possible to reduce the complexity of interaction with humans to an explicit decision on the topology class of interaction (along the lines of~\cite{mavrogiannis_multi-agent_2019}) that simplifies the planning problem. Additionally, T-MPC\texttt{++} can possibly be extended to $3$D navigation in dynamic environments, for instance using higher-dimensional H-signatures, and the guidance planner could be extended to incorporate non-Gaussian (i.e., multi-modal) uncertainty in obstacle motion (e.g.,~\cite{de_groot_scenario-based_2023}). %

\section{CONCLUSION}\label{sec:conclusion}
We presented in this paper a two-fold planning approach to address the inherent local optimality of optimization-based planners. Our planner consisted of a high-level global planner and a low-level optimization-based planner. By accounting for the topology classes of trajectories in the dynamic free space, we generated trajectories with distinct passing behaviors that we then used to guide several local optimization-based planners in parallel.

We simulated a mobile robot navigating among pedestrians and showed that the proposed guided planner resulted in faster and more consistent robot motion than existing planners, including a state-of-the-art topology-guided planner. We qualitatively observed the same improvement in the real world, where we navigated successfully among five pedestrians.

In future work, we aim to deploy the proposed method, considering uncertainty in obstacle motion, on a self-driving vehicle navigating in urban environments.

\begin{appendices}
\section{HOMOTOPY COMPARISON}\label{ap:homotopy_comparison}
This appendix details and compares three implementations of homotopy comparison function~\eqref{eq:homotopy_comparison} for $2$D motion planning in dynamic environments. %

\subsection{H-signature}\label{sec:homology}
The H-signature~\cite{bhattacharya_topological_2012} approximates homotopy classes by homology classes, formally defined as follows.
\begin{definition}
    \cite{bhattacharya_topological_2012}~(Homologous Trajectories) Two trajectories $\b{\tau}_1, \b{\tau}_2 \in T$ connecting the same start and end points $\b{x}_s$ and $\b{x}_g$ respectively, are homologous iff $\b{\tau}_1$ together with $\b{\tau}_2$ (the latter in the opposite direction) forms the complete boundary of a 2-dimensional manifold embedded in $\mathcal{X}$ not containing or intersecting any of the obstacles.
\end{definition}
\begin{figure}[t]
    \centering
    \includegraphics[width=0.35\textwidth]{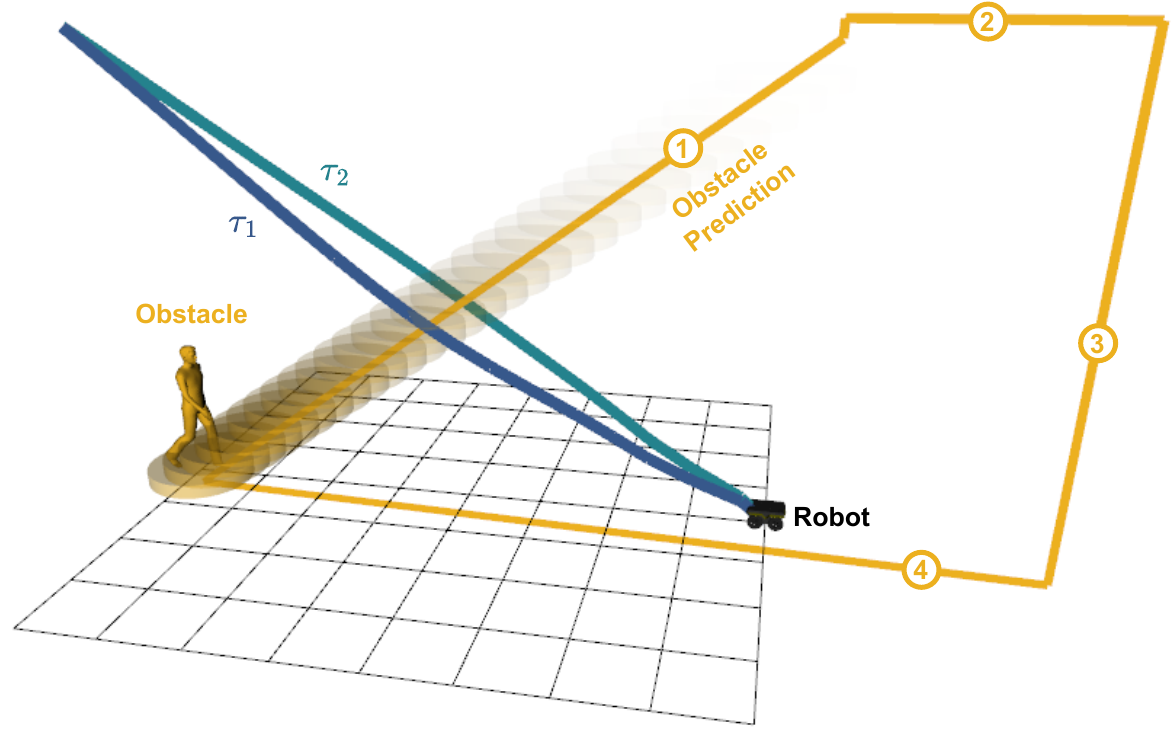}
    \caption{Illustrating example for the H-signature~\cite{bhattacharya_topological_2012}. The two trajectories $\b{\tau}_1$ (blue) and $\b{\tau}_2$ (green) are in distinct homology classes as the loop that they form contains obstacle skeleton (orange). In practice, link (3) is placed far away.} %
    \label{fig:h-signature}
\end{figure}%

If two trajectories are homotopic, they are homologous. The reverse does not hold. To compute the H-signature in the $3$D space composed of $x,y$-position and time, each obstacle and its prediction is virtually modeled as a current-carrying wire (see Fig.~\ref{fig:h-signature}). The H-signature of Obstacle $j$, $h^j(\b{\tau})$ is defined by the virtual magnetic field resulting from Obstacle $j$'s current loop, integrated over trajectory $\b{\tau}$. If two trajectories $\b{\tau}_1, \b{\tau}_2$ (see Fig.~\ref{fig:h-signature}) enclose the loop of obstacle $j$, then $h^j(\b{\tau}_1) \neq h^j(\b{\tau}_2)$. Hence, two trajectories are equivalent if $h^j(\b{\tau}_1) = h^j(\b{\tau}_2) \ \forall j$ and are distinct otherwise. 
We compute the H-signature for a finite-time state space, by constructing a $1$D skeleton of the obstacle prediction that is looped outside of the workspace and time horizon (see Fig.~\ref{fig:h-signature}). The skeleton is composed of the following lines. \circled{1} The obstacle's prediction for $0 < t < T$. \circled{2} A line upwards to $t = T+\epsilon$ where $\epsilon>0$ is a small constant and a line going outside of the workspace. \circled{3} A line down to $t = -\epsilon$. \circled{4} A line to the obstacle position at $t = 0$. This skeleton ensures that trajectories can only enclose the predicted obstacle motion for $0 < t < T$.

We assume that obstacle trajectories are piecewise linear (e.g., discrete-time trajectories). The integration of the magnetic field $\b{B}$ can then be computed analytically (see \cite{bhattacharya_topological_2012}) per segment $i$ of Obstacle $j$'s skeleton $\overline{\b{o}_i^j \b{o}_i^{j'}}$ as follows:
\begin{align*}
    &\b{p} = \b{o}_i^j - \b{r}, \ \b{p}' = \b{o}_i^{j'} - \b{r}, \ \b{d} = \frac{(\b{o}_i^{j'} - \b{o}_i^j) \times (\b{p}\times\b{p}')}{||\b{o}_i^{j'} - \b{o}_i^j||^2},\\
    &\b{\Phi}(\b{o}_i^j, \b{o}_i^{j'}, \b{r}) = \frac{1}{||\b{d}||^2}\left(\frac{\b{d}\times\b{p}'}{||\b{p}'||} - \frac{\b{d} \times \b{p}}{||\b{p}||}\right),\\
    &\b{B}(\b{r}) = \frac{1}{4\pi}\sum_{i = 0}^{I} \b{\Phi}(\b{o}_i^j, \b{o}_i^{j'}, \b{r}),
\end{align*}
with $I$ the number of segments in Obstacle $j$'s skeleton. The integral $\int_{\b{l}} \b{B}(\b{r})d\b{r}$ over looped robot trajectories $\b{l}$ yields $1$ if the obstacle is enclosed and $0$ otherwise. To compare trajectories that reach different goals in our guidance planner, we connect their end points directly at $t=T$ with an additional line. We use the GSL library~\cite{galassi_gnu_2019} to perform the integration and cache computed H-signature for each trajectory to prevent re-computation.

\subsection{Winding Number}
The winding number~\cite{berger_topological_2001} is a topological invariant that indicates how the robot and obstacle $j$ are rotated around each other. It is computed as follows. The relative position of obstacle $j$ to the robot for time step $k$ is $\b{d}_k^j = \b{p}_k - \b{o}_k^j$. The relative angle $\angle \theta_k^j$ is the angle of $\b{d}_k^j$ in a fixed global frame. Between time steps $k$ and $k+1$, the relative angle changes by $\Delta \theta_k^j = \theta_{k + 1}^j - \theta_k^j$. The winding number accumulates these changes over all time steps, $\lambda(\b{\tau}, \b{o}^j) = \frac{1}{2\pi}\sum_{k = 1}^{N} \Delta \theta_k^j.$
The sign of the winding number $\lambda$ indicates the passing direction, its magnitude denotes passing progress. We consider a trajectory to pass obstacle $j$ if $|\lambda^j| \geq \lambda_{\textrm{pass}}$, where by default $\lambda_{\textrm{pass}} = \frac{1}{4\pi}$. We consider two trajectories distinct if there exists at least one obstacle that the trajectories pass on different sides and consider them equivalent otherwise\footnote{Future work could also use winding numbers to distinguish between passing and non-passing trajectories.}. We cache computed winding numbers to prevent recomputation.

\subsection{Universal Visibility Deformation}
Universal Visibility Deformation (UVD) was proposed for static obstacle avoidance in $3$D and therefore does not exactly capture the local optima for collision avoidance in $2$D dynamic environments. Two trajectories are in the same UVD class if points along the trajectories can be connected, without intersecting with obstacles.

\begin{definition}{\normalfont{\cite{zhou_robust_2020}}}
    Two trajectories $\b{\tau}_1(s), \b{\tau}_2(s)$ parameterized by $s \in [0, 1]$ and satisfying $\b{\tau}_1(0) = \b{\tau}_2(0)$, $\b{\tau}_1(1) = \b{\tau}_2(1)$, belong to the same uniform visibility deformation class, if for all $s$, line $\overline{\b{\tau}_1(s)\b{\tau}_2(s)}$ is collision-free.
\end{definition}%
In practice, we check collisions for $s$ at discrete intervals along the trajectories.

\begin{table}[b]
\centering
\caption{Comparison between homotopy comparison implementations. Notation follows that of Table~\ref{tab:pedsim_cvpredictions}.}
\resizebox{0.49\textwidth}{!}{\begin{tabular}{|l|l|c|c|c|}
\hline\textbf{\#} & \textbf{Method} & \textbf{Dur. [s]} & \textbf{Safe (\%)} & \textbf{Homotopy Comparison Time [ms]} \\\hline
\multirow{3}{*}{\specialcell{Homotopy\\Comparison}} & Winding Angles & \sigs{21.2} (0.9) & \textbf{93} & \textbf{0.3} (0.7) \\
&H-Signature & 21.2 (1.0) & 92 & 2.1 (4.0) \\
&UVD & \textbf{21.1} (0.9) & 88 & 3.7 (8.5) \\\hline
\end{tabular}}
\label{tab:homotopy_comparison}
\end{table}

\subsection{Comparison}
We compare the homotopy comparison functions in simulation on the scenario of Sec.~\ref{sec:results_crowded} over $100$ experiments. Table~\ref{tab:homotopy_comparison} indicates that UVD degrades navigation performance, likely because it is not designed for dynamic environments and may lead to duplicate trajectories in practice. The H-signature and winding numbers show similar navigation performance. Winding numbers are computationally more efficient but require a minimum passing angle to be tuned. Since the H-signature generalizes to higher dimensions and both methods are still real-time, we opted to use the H-signature in this paper.

\end{appendices}

\bibliographystyle{IEEEtran}
\bibliography{references}

\begin{IEEEbiography}[{\includegraphics[width=1in,height=1.25in,clip,keepaspectratio]{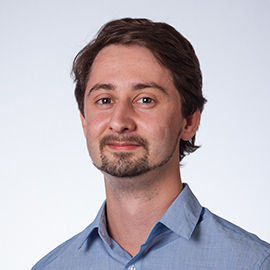}}]{Oscar de Groot} received both the B.Sc. degree in electrical engineering, in 2016, and the M.Sc. degree in systems \& control, in 2019, from the Delft University of Technology, Delft, The Netherlands. He is currently pursuing a Ph.D. in motion planning for autonomous vehicles in urban environments at the department of Cognitive Robotics at the Delft University of Technology. His research interests include probabilistic safe motion planning, scenario optimization, model predictive control and self-driving vehicles.
\end{IEEEbiography}
\begin{IEEEbiography}[{\includegraphics[width=1in,height=1.25in,clip,keepaspectratio]{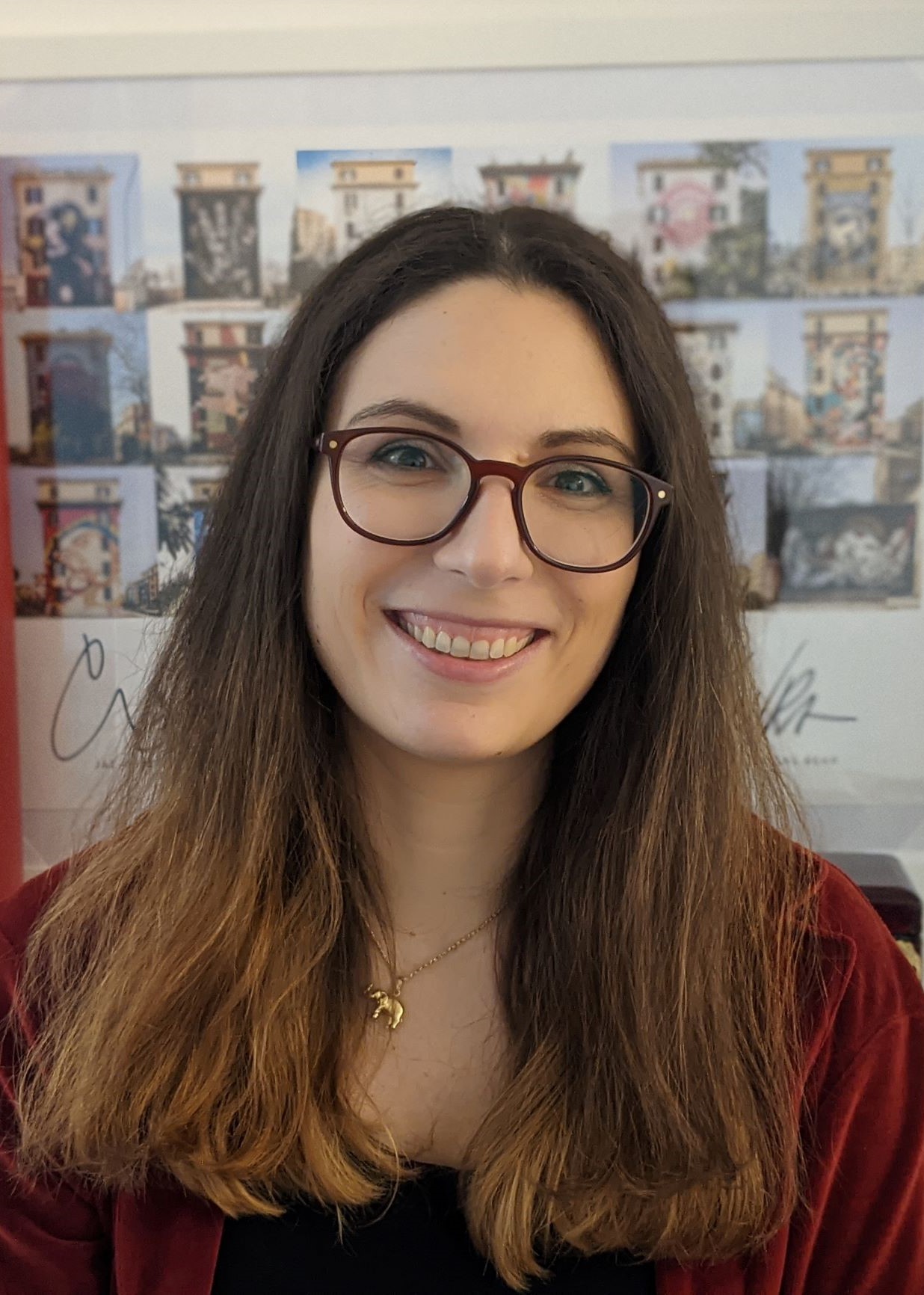}}]{Laura Ferranti} received her PhD from Delft University of Technology, Delft, The Netherlands, in
2017. She is currently an assistant professor in the Cognitive
Robotics (CoR) Department, Delft University of Technology, Delft, The
Netherlands. She is the recipient of an NWO Veni Grant from The
Netherlands Organisation for Scientific Research (2020), and of the Best Paper
Award in Multi-robot Systems at ICRA 2019.
Her research interests include optimization and optimal control,
model predictive control, reinforcement learning, embedded optimization-based
control with application in flight control, maritime transportation, robotics, and automotive.   
\end{IEEEbiography}
\begin{IEEEbiography}[{\includegraphics[width=1in,height=1.25in,clip,keepaspectratio]{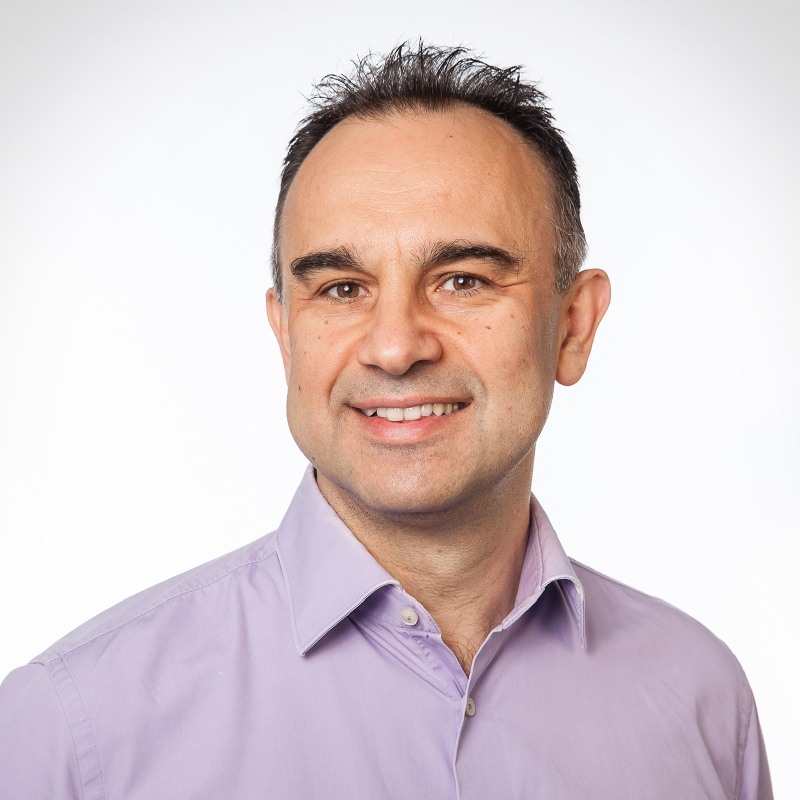}}]{Dariu Gavrila}
 received the Ph.D. degree in computer science from Univ. of Maryland at College Park, USA, in 1996. From 1997 until 2016, he was
with Daimler R\&D, Ulm, Germany, where he became a Distinguished Scientist. He led the vision-based pedestrian detection research, which was commercialized 2013-2014 in various Mercedes-Benz models. In 2016, he moved to TU Delft, where he since heads the Intelligent Vehicles group as a Full Professor. His current research deals with sensor-based detection of humans and analysis of behavior in the context of self-driving vehicles. He was awarded the Outstanding Application Award 2014 and the Outstanding Researcher Award 2019, both from the IEEE Intelligent Transportation Systems Society.
\end{IEEEbiography}
\begin{IEEEbiography}[{\includegraphics[width=1in,height=1.25in,clip,keepaspectratio]{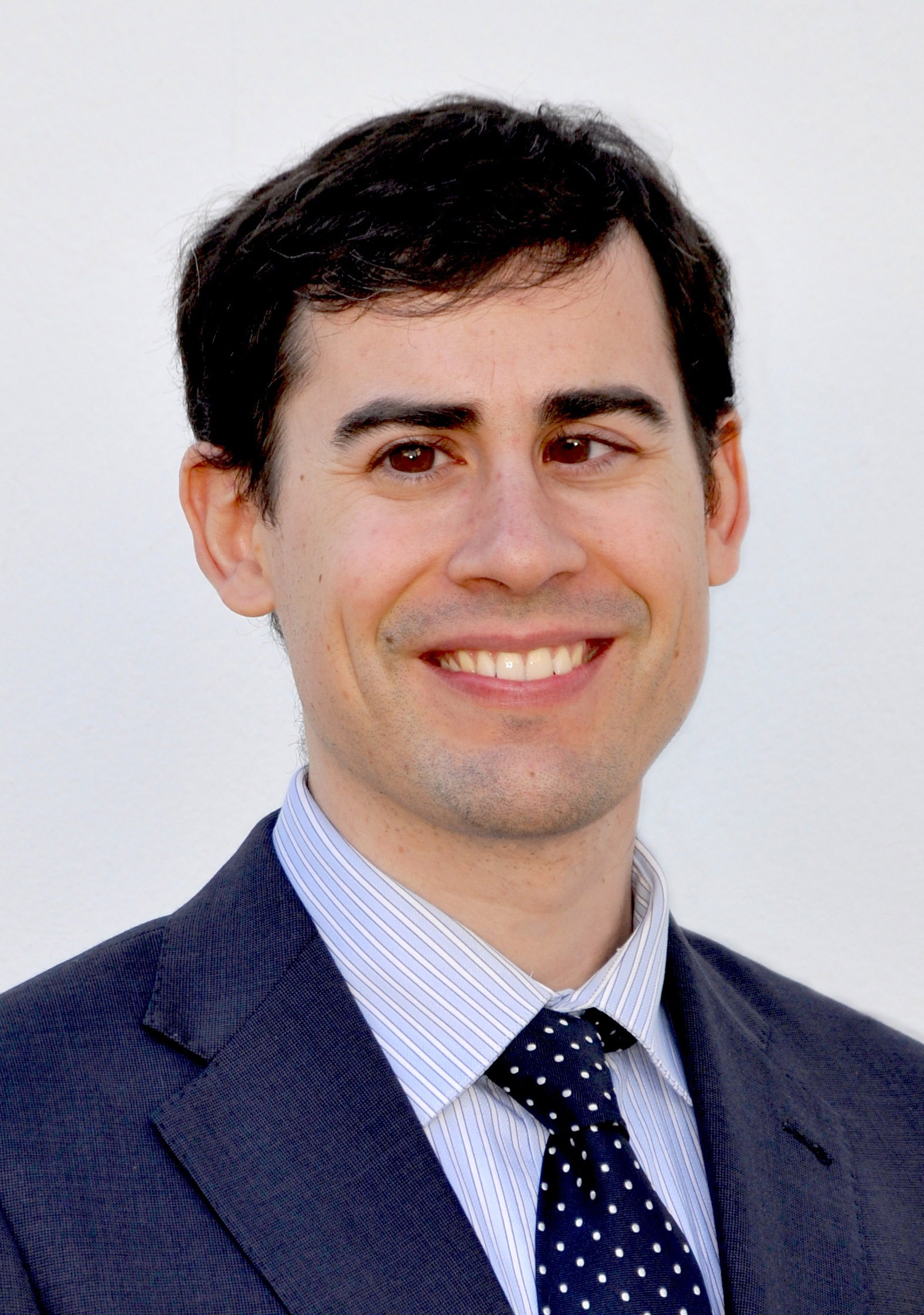}}]{Javier Alonso-Mora}
 is an Associate Professor at the Cognitive Robotics department of the Delft University of Technology, and a Principal Investigator at the Amsterdam Institute for Advanced Metropolitan Solutions (AMS Institute).
Before joining TU Delft, Dr. Alonso-Mora was a Postdoctoral Associate at the Massachusetts Institute of Technology (MIT). He received his Ph.D. degree in robotics from ETH Zurich,

His main research interest is in navigation, motion planning and control of autonomous mobile robots, with a special emphasis on multi-robot systems, on-demand transportation and robots that interact with other robots and humans in dynamic and uncertain environments. He is the recipient of an ERC Starting Grant (2021), the ICRA Best Paper Award on Multi-Robot Systems (2019), an Amazon Research Award (2019) and a talent scheme VENI Grant from the Netherlands Organisation for Scientific Research (2017).
\end{IEEEbiography}

\end{document}